\documentclass[10pt]{article}
\author{}
\usepackage{tmpap}
\usepackage[sectionbib,square]{natbib}

\definecolor{DarkGreen}{rgb}{0.1,0.5,0.1}
\definecolor{DarkRed}{rgb}{0.5,0.1,0.1}
\definecolor{DarkBlue}{rgb}{0.1,0.1,0.5}
\definecolor{RoyalBlue}{RGB}{0,100,170}
\definecolor{SapphireBlue}{RGB}{0, 100, 200}
\usepackage{hyperref}
\hypersetup{
    unicode=false,          % non-Latin characters in Acrobat¿s bookmarks
    pdftoolbar=true,        % show Acrobat toolbar?
    pdfmenubar=true,        % show Acrobat menu?
    pdffitwindow=false,      % page fit to window when opened
    pdfnewwindow=true,      % links in new window
    colorlinks=true,       % false: boxed links; true: colored links
    linkcolor=blue,          % color of internal links
    citecolor=blue,        % color of links to bibliography
    filecolor=blue,      % color of file links
    urlcolor=DarkBlue,          % color of external links
}

\usepackage{enumitem}% http://ctan.org/pkg/enumitem

\newtheorem{theorem}{Theorem}
\newtheorem*{theorem*}{Theorem}
\newtheorem{lemma}{Lemma}
\newtheorem{corollary}{Corollary}
\newtheorem{proposition}{Proposition}
\newtheorem{definition}{Definition}
\newtheorem{remark}{Remark}
\newtheorem*{question}{Question}

\newcommand{\numberthis}{\addtocounter{equation}{1}\tag{\theequation}}
\newcommand{\argmin}{\mathop{\mathrm{argmin}}}

\def\R{\mathbb{R}}

\def\N{\mathbb{N}}

\def\divergence{\mathrm{div}}

\def\eg{\emph{e.g.}}

\DeclareMathOperator{\spn}{span}
\DeclareMathOperator{\nn}{nn}

\def\h{H^1_0(\Omega)}
\def\hh{H^1(\Omega)}
\def\ch{\mathcal{H}}
\def\l2{L^2(\Omega)}

\def\linf{L^\infty(\Omega)}

\def\pc{C_p}
\def\Keigenfunctions{\Phi_k}
\def\Keigenfunctionsmath{\Phi_k}
\def\Keigenfunctionsapprox{\tilde{\Phi}_k}

\def\numassumptions{(i)-(iii)}
\def\tl{\tilde{L}}
\def\ta{\tilde{A}}
\def\tc{\tilde{c}}

\def\tlambda{\tilde{\lambda}}
\def\tvarphi{\tilde{\varphi}}
\def\fkeigen{f_{\spn}}
\def\fkeigenapprox{\tilde{f}_{\spn}}
\def\fnn{f_{\nn}}

\def\invkappa{\frac{\tlambda_{k} - \tlambda_{1}}{\tlambda_{k} + \tlambda_{1}}}

\def\fnumparams{N_f}

\def\ustar{u^\star}
\def\uspansolution{u_{\spn}^\star} %solution of Lu = f
\def\uapproxsolution{\tilde{u}_{\spn}^\star} %solution of \tilde{L}u = \tilde{f}_span
\def\usemieigen{\tilde{u}} % final solution of Lu = \tilde{f}_span
\def\hu{\hat{u}} % neural network approximation for u

\def\errordaviskahan{\frac{\|f\|_{\l2}}{\gamma - \delta }}
\def\gammavalue{\frac{1}{\lambda_k} - \frac{1}{\lambda_{k+1}}}
\def\deltavalue{\max\left\{\frac{\epsilon_A}{m}, \frac{\epsilon_c}{\zeta}\right\}}

\def\errorspan{
\frac{\epsilon_{\spn}}{\lambda_1}
+ \frac{\delta}{\lambda_1} \errordaviskahan
+ \delta \|\uapproxsolution\|_{\l2}
}

\def\Cvalue{
(2d^2 + 1)\max\left\{\max_{\alpha:|\alpha|\leq 3}\max_{i,j}\|\partial^{\alpha}a_{ij}\|_{\linf},
\max_{\alpha:|\alpha|\leq 2}\|\partial^\alpha c\|_{\linf}\right\}}

\def\errorf{f}
\def\errorspanstar{\frac{\|\errorf\|_{\l2}\delta}{\gamma - \delta}}

\def\errorLntildeLnold{\delta}
\def\errorLntildeLn{2\delta}

\def\learningrate{\frac{2}{\tlambda_1 + \tlambda_k}}
\def\learningratenotilde{\frac{2}{\lambda_1 + \lambda_k}}

\def\constantdeltagamma{\frac{\delta}{\gamma - \delta}}

\def\constantnumber{4}

\definecolor{prune}{rgb}{0.44, 0.11, 0.11}
\definecolor{myblue}{rgb}{0, .5, 1}
\definecolor{maroon}{rgb}{0.5450, 0, 0}
\definecolor{darkred}{rgb}{0.5450, 0, 0}
\definecolor{RoyalBlue}{RGB}{0,100,170}
\definecolor{DarkBlue}{RGB}{20,70,200}
\definecolor{peach}{rgb}{1, 0.56, 0.56}
\definecolor{NotionGreen}{RGB}{15,123,108}
\definecolor{NotionOrange}{RGB}{217,115,13}
\definecolor{NotionRed}{RGB}{224,62,62}
\definecolor{MontrealBlue}{RGB}{0, 30, 98}

\include{defs}
\def\shownotes{1}  %set 1 to show author notes
\ifnum\shownotes=1
\newcommand{\authnote}[2]{{$\ll$\textsf{\footnotesize #1: #2}$\gg$}}
\else
\newcommand{\authnote}[2]{}
\fi

\title{
Parametric Complexity Bounds for  Approximating PDEs
\\with Neural Networks
}
\author{Tanya Marwah,
Zachary C. Lipton,
Andrej Risteski \\ 
Machine Learning Department, Carnegie Mellon University \\ 
\texttt{\{tmarwah, zlipton, aristesk\}@andrew.cmu.edu}}
\date{}

\begin{document}

\maketitle
\begin{abstract}
    % Recent empirical results show
% that 
Recent experiments have shown
that deep networks can approximate solutions
to high-dimensional PDEs, 
seemingly escaping the curse of dimensionality.
However, questions regarding
the theoretical basis 
for such approximations,
including the required network size, 
% including the number of parameters required, 
remain open.
In this paper, we investigate the 
representational power of neural networks 
for approximating solutions to linear elliptic PDEs 
with Dirichlet boundary conditions.
We prove that when a PDE's coefficients 
are representable by small 
neural networks,
the parameters required to approximate its solution 
scale polynomially with the input dimension $d$
% and are proportional to the parameter counts
and proportionally to the parameter counts 
of the coefficient networks.
% Our proof is based on constructing a neural network 
% Our proof technique
To this we end, we develop a proof technique that
simulates gradient descent
(in an appropriate Hilbert space)
by growing a neural network architecture
whose iterates each participate
as sub-networks in their 
(slightly larger) successors, 
and converge
to the solution of the PDE.
% ==== Old version start ====
% Our proof is based on constructing a neural network which simulates gradient descent in 
% an appropriate Hilbert space which converges to the solution of the PDE. 
% Moreover, we bound the size of the neural network needed to represent 
% each iterate in terms of the neural network representing the previous iterate,
% resulting in a final network whose parameters depend polynomially on $d$
% and does not depend on the volume of the domain.
% ==== Old version end ====
% We bound the size 
% of the neural network needed 
% to represent each iterate 
% in terms of the neural network 
% representing the previous iterate.
We bound the size of the solution,
% required network,
showing a polynomial dependence on $d$
and no dependence
on the volume of the domain.

\end{abstract}

\section{Introduction} 

A partial differential equation (PDE) 
relates a multivariate function 
defined over some domain 
to its partial derivatives.
Typically, one's goal is to solve 
for the (unknown) function,
often subject to additional constraints,
such as the function's value 
on the boundary of the domain.
PDEs are ubiquitous in both 
the natural and social sciences,
where they model such diverse processes 
as heat diffusion \citep{crank1947practical, ozicsik2017finite},
fluid dynamics \citep{anderson1995computational, temam2001navier}, 
and financial markets \citep{black1973pricing, ehrhardt2008fast}.
Because most PDEs of interest 
lack closed-form solutions,
computational approximation methods
remain a vital and an active field of research \citep{ames2014numerical}.
For low-dimensional functions, 
dominant approaches include the finite differences 
and finite element methods \citep{leveque2007finite},
which discretize the domain.
% which operate on discretizations of the domain.
After partitioning the domain into a \emph{mesh},
these methods solve for the function value at its vertices.
However, %such mesh-based methods 
these techniques
scale exponentially with the input dimension, 
rendering them unsuitable 
for high-dimensional problems. 

Following breakthroughs in deep learning 
for approximating high-dimensional functions
in such diverse domains as computer vision 
\citep{NIPS2012_c399862d, radford2015unsupervised}
and natural language processing 
\citep{bahdanau2014neural, devlin2018bert,vaswani2017attention},
a burgeoning line of research
leverages neural networks 
to approximate solutions to PDEs.
This line of work has produced 
promising empirical results for common PDEs 
such as the Hamilton-Jacobi-Bellman 
and Black-Scholes equations 
\citep{han2018solving, grohs2018proof, sirignano2018dgm}.
Because they do not explicitly discretize the domain,
and given their empirical success 
on high-dimensional problems, 
these methods \emph{appear} 
not to suffer the curse of dimensionality.
However, these methods are not well understood theoretically,
leaving open questions about when they are applicable,
what their performance depends on,
and just how many parameters are required
to approximate the solution to a given PDE.

Over the past three years,
several theoretical works 
have investigated questions 
of representational power 
under various assumptions.
% \citep{khoo2017solving} 
% \blue{parameterize the solution of a deterministic PDE
% using a neural network and provide 
% parametric results that depend exponentially on the input dimension.}
% theoretical guarantees on their parametric complexity, 
% albeit with size exponential in the input dimension.}
% prove a kind of 
% universal approximation power of neural networks, 
Exploring a variety of settings,
\citet{kutyniok2019theoretical},
\citet{grohs2018proof},
and \citet{jentzen2018proof}, 
proved that the number of parameters 
required to approximate a solution to a PDE
exhibits a less than exponential dependence
on the input dimension 
for some special parabolic PDEs
that admit straightforward analysis.
% for which the analysis is much more straightforward.
%However, these results assume 
%the entire $\R^d$ as their domain
%and do not consider boundary conditions.
\citet{grohs2020deep} consider elliptic PDEs with
Dirichlet boundary conditions.
However, their rate depends 
on the volume of the domain, 
and thus can have an implicit 
exponential dependence on dimension 
(e.g., consider a hypercube 
with side length greater than one). 
%limiting its applicability to domains 
%where volume does not scale with the dimension,
%\eg, a hypercube of side length one.

In this paper, 
we focus on linear elliptic PDEs 
with Dirichlet boundary conditions,
which are prevalent in science and engineering 
(\eg, the Laplace and Poisson equations).   
% Most commonly, 
Notably, linear elliptic PDEs 
define the steady state of processes 
like heat diffusion and fluid dynamics. 
% We provide an answer to the following:
% Our work addresses the following question:
Our work asks:
\begin{question}
How many parameters suffice 
to approximate the solution 
to a linear elliptic PDE 
up to a specified level of precision
using a neural network?
\end{question}
 
We show that when the coefficients of the PDE
are expressible as small neural networks 
(note that PDE coefficients are functions), 
the number of parameters required 
% to express an approximate solution to the PDE
to approximate the PDE's solution 
is proportional to 
the number of parameters required 
to express the coefficients.
Furthermore, we show that the number of parameters 
depends polynomially on the dimension
and does not depend upon the volume of the domain.

\section{Overview of Results}
% We will consider PDEs that
% take the following form: 
% Our work addresses linear elliptic PDEs, formally defined as follows:
To begin, we formally define linear elliptic PDEs.
\begin{definition}[Linear Elliptic PDE \citep{evans1998partial}]
\label{eq:main_pde}
Linear elliptic PDEs
with Dirichlet boundary condition
can be expressed in the following form:
% in the following form:
% takes the following form,
% is a 
% partial differential equation 
% of form
\[\left\{
\begin{aligned}
    \left(Lu\right)(x) \equiv \left(-\text{div}\left(A \nabla u\right) + cu\right)(x) &= f(x), 
    \forall x \in \Omega, \\
     u(x)&= 0, \forall x \in  \partial \Omega,
\end{aligned} \right.\]
where $\Omega \subset \R^d$ 
is a bounded open set 
with a boundary $\partial \Omega$.
Further, for all $x \in \Omega$, 
$A: \Omega \to \mathbb{R}^{d \times d}$ 
is a matrix-valued function, s.t. 
$A(x) \succ 0$, and $c: \Omega \to \R$, s.t. $c(x) > 0$.
\footnote{
% Here $\text{div}$ denotes the divergence operator, which given a vector field $F:\mathbb{R}^d \to \mathbb{R}^d$ outputs
% $\text{div}(F) = 
% \nabla \cdot F = \sum_{i=1}^d \frac{\partial F_i}{\partial x_i}$
Here, $\text{div}$ denotes the divergence operator.
Given a vector field $F:\mathbb{R}^d \to \mathbb{R}^d$,
$\text{div}(F) = 
\nabla \cdot F = \sum_{i=1}^d \frac{\partial F_i}{\partial x_i}$
}
\end{definition}

We refer to $A$ and $c$ 
as the \emph{coefficients} of the PDE. 
The divergence form in Definition \ref{eq:main_pde}
is one of two canonical ways to define 
a linear elliptic PDE \citep{evans1998partial} 
and is convenient for several technical reasons
(see Section~\ref{section:prelim_notation_definitions}). 
% The boundary condition being such that the solution 
% takes a constant value (here $0$) 
% is also without loss of generality. 
% Without loss of generality,
The Dirichlet boundary condition
states that the solution 
takes a constant value (here $0$) 
on the boundary $\partial \Omega$.

% is also without loss of generality. 
% and for convenience of exposition. 

Our goal is to express
the number of parameters 
required to approximate the solution of a PDE 
in terms of 
% the number of parameters
those
required to approximate its coefficients $A$ and $c$.
% 
% (in terms of number of parameters 
% of a neural network that approximates it) 
% when the coefficients $A$ and $c$ 
% can be expressed as neural networks 
% with small number of parameters. 
% We are interested in the complexity 
% of the solution of the PDE 
% (in terms of number of parameters 
% of a neural network that approximates it) 
% when the coefficients $A$ and $c$ 
% can be expressed as neural networks 
% with small number of parameters. 
% (For technical details,
% (see Section~\ref{section:main_result}.)
% Informally, our result states: % the following: 
Our key result shows:
\begin{theorem*}[Informal]
    If the coefficients $A, c$ 
    and the function $f$ 
    are approximable by neural networks 
    with at most $N$ parameters,
    the solution $u^\star$ to the PDE 
    in Definition \ref{eq:main_pde} 
    is approximable by a neural network 
    with $O\left(\mbox{poly}(d)N\right)$ parameters.
\end{theorem*}
% \noindent
% The formal statement is provided
% in Section~\ref{section:main_result}. 
% Our results provide supporting theory 
% that may help
This result, 
formally expressed in Section~\ref{section:main_result},
may help to explain
the practical efficacy of neural networks 
in approximating solutions 
to high-dimensional PDEs 
with boundary conditions 
% as shown in 
\citep{sirignano2018dgm, li2020fourier}.
To establish this result,
we develop a constructive proof technique
that simulates gradient descent 
(in an appropriate Hilbert space)
through the very architecture of a neural network.
Each iterate, given by a neural network,
is subsumed into the (slightly larger)
network representing the subsequent iterate.
The key to our analysis is to bound 
both (i) the growth in network size 
across consecutive iterates; 
and (ii) the total number of iterates required.
% Namely, we (i) show that each iterate can be approximated 
% by a small neural network slightly larger 
% than the neural network 
% approximating the previous iterate,
% as well as bound the total number of iterates. 
% \Anote{add something about sample complexity here}

\paragraph{Organization of the paper}
We introduce the required notation 
along with some mathematical preliminaries on PDEs in Section~\ref{section:prelim_notation_definitions}.
The problem setting and formal statement 
of the main result are provided 
in Section~\ref{section:main_result}.
Finally, we provide the proof of the main result in  Section~\ref{section:proof_of_main_result}.

\section{Prior Work}
Among the first papers 
to leverage neural networks
to approximate solutions to PDEs 
with boundary conditions 
are \citet{lagaris1998artificial}, 
\citet{lagaris2000neural},
and \citet{malek2006numerical}.
However, these methods discretize the input space 
and thus are not suitable 
for high-dimensional input spaces.
More recently, mesh-free neural network approaches 
have been proposed for high-dimensional PDEs 
\citep{han2018solving, raissi2017physics,raissi2019physics},
achieving impressive empirical results in various applications.
\citet{sirignano2018dgm} design a loss function 
that penalizes failure to satisfy the PDE, 
training their network on minibatches 
sampled uniformly from the input domain.
They also provide a universal approximation result,
% that shows 
showing
that for sufficiently regularized PDEs, 
there exists a multilayer network 
that approximates its solution. 
However, they do not comment on 
the complexity of the neural network 
or how it scales with the input dimension.
\citet{khoo2017solving} also prove 
universal approximation power, 
albeit with networks of size 
exponential in the input dimension. 
% Recent work from 
Recently, \citet{grohs2018proof, jentzen2018proof} 
provided a better-than-exponential dependence 
on the input dimension 
for some special parabolic PDEs, 
for which the simulating a PDE solver 
by a neural network is straightforward.

% Relatedly
Several recent works~\citep{bhattacharya2020model, kutyniok2019theoretical, li2020neural, li2020fourier}
% present empirical results showing 
show (experimentally)
that a single neural network 
can solve for an entire family of PDEs. 
They approximate the map 
from a PDE's parameters 
to its solution,
% using a neural network,
potentially 
%saving practitioners
avoiding
the trouble of retraining for every set of coefficients. 
% This can potentially save practitioners 
% the trouble of retraining a new network 
% each time the coefficients of the PDE change.
Among these, only \citet{kutyniok2019theoretical}
provides theoretical grounding. 
However, they assume the existence 
of a finite low-dimensional space
with basis functions
that can approximate this parametric map---and 
it is unclear when this would obtain. 
Our work proves the existence of such maps,
under the assumption that the family of PDEs 
has coefficients described by neural networks 
with a fixed architecture 
(Section~\ref{section:discussion}).

In the work most closely related to ours, 
\cite{grohs2020deep} provides approximation rates 
polynomial in the input dimension $d$ 
for the Poisson equation 
(a special kind of linear elliptic PDE) 
with Dirichlet boundary conditions.
They introduce a walk-on-the-sphere algorithm, 
which simulates a stochastic differential equation 
that can be used to solve a Poisson equation
with Dirichlet boundary conditions 
(see, e.g.,  \citet{oksendal2013stochastic}'s Theorem 9.13).
The rates provided in \citet{grohs2020deep} 
depend on the volume of the domain, 
% and thus have an implicit dependence
and thus depend, implicitly, exponentially on the input dimension $d$.
%However, 
% as they show,
%their construction 
%avoids the curse of dimensionality
%by considering domains
% for domains
%where the volume 
%remains constant in dimensions, 
%e.g., a hypercube with side length one.
Our result considers
the boundary condition for the PDE
and is
independent of the volume of the domain.
Further, we note that our results are defined 
for a more general linear elliptic PDE, 
of which the Poisson equation is a special case.

\section{Notation and Definitions}
\label{section:prelim_notation_definitions}
We now introduce several key concepts 
from PDEs and some notation.
For any open set $\Omega \subset \R^d$, 
we denote its boundary by $\partial \Omega$ 
and denote its closure by 
$\bar{\Omega} := \Omega \cup \partial \Omega$.
By $C^0(\Omega)$, we denote the space 
of real-valued continuous functions 
defined over the domain $\Omega$. 
Furthermore, for $k \in \N$, 
% a function $g \in C^k(\Omega)$ if
a function $g$ belongs to $C^k(\Omega)$ if all partial derivatives $\partial^\alpha g$ exist and are continuous for any multi-index $\alpha$, 
such that $|\alpha| \leq k$.
Finally, a function $g \in C^\infty(\Omega)$ 
if $g \in C^k(\Omega)$ for all $k \in \N$.
% Next, we define several function spaces
% that are essential to our analysis:
% Next, we define several function spaces
% that are essential to our analysis:
Next, we define several relevant function spaces:

\begin{definition}
For 
% a
any $k \in \N \cup \{\infty\}$,
$C^k_0(\Omega):= \{g: g \in C^k(\Omega) 
, \overline{\mbox{supp}(g)} \subset \Omega\}$. 
\end{definition}

\begin{definition}
For a domain $\Omega$, 
the function space $L^2(\Omega)$ consists of all functions $g: \Omega \to \R$, s.t. 
% $$\l2 := \left\{g: \|g\|_{\l2} < \infty\right\},$$
%\l2 := \left\{g: \Omega \to \R |  
$\|g\|_{\l2} < \infty$
where $\|g\|_{\l2} = \left(\int_\Omega |g(x)|^2 dx\right)^{\frac{1}{2}}$.
This function space is equipped 
with the inner product
$$\langle g, h\rangle_{\l2} = \int_\Omega g(x) h(x) dx.$$
\end{definition}

\begin{definition}
For a domain $\Omega$ and a function 
$g: \Omega \rightarrow \R$, 
the function space $L^{\infty}(\Omega)$ 
is defined analogously, where 
% $\|g\|_{L^{\infty}(\Omega)} = \inf \{c \geq 0: |g(x)| \geq 0 \text{ for almost all } x \in \Omega \}$.
$\|g\|_{L^{\infty}(\Omega)} = \inf \{c \geq 0: |g(x)| \leq c \text{  for almost all } x \in \Omega \}$.
\end{definition}

\begin{definition}
    For a domain $\Omega$ and $m \in \N$, 
    we define the Hilbert space $H^m(\Omega)$ as
    $$ H^m(\Omega) := \{g: \Omega \rightarrow \R: 
    \partial^\alpha g \in L^2(\Omega), \; \forall \alpha \; \text{s.t.} \;\; |\alpha| \leq m\}$$
    Furthermore, $H^m(\Omega)$
    is equipped with the inner product,
    $\langle g, h\rangle_{H^m(\Omega)} = 
    \sum_{|\alpha|\leq m} 
    \int_\Omega (\partial^\alpha g)(\partial^\alpha h) dx
    $
    and the corresponding norm
    $$
    \|g\|_{H^m(\Omega)} = 
    \left(\sum_{|\alpha|\leq m} 
    \|\partial^\alpha g \|^2_{L^2(\Omega)}\right)^{\frac{1}{2}}.
    $$
\end{definition}
% \blue{
% Further,
% $H_0^m(\Omega)$ 
% consits the set of functions 
% $g \in H^m(\Omega)$ 
% such that $g(x) = 0$ for all $x \in \partial \Omega$ 
% (refer to Section~\ref{section:notation_appendix} for a more formal definition).
% This space (particularly with $m=1$)
% is often useful when analyzing elliptic PDEs 
% with Dirichlet boundary conditions.
% }

\begin{definition}
    \label{def:closure}
    The closure of $C^\infty_0(\Omega)$ in $H^m(\Omega)$ 
    is denoted by $H^m_0(\Omega)$.
\end{definition}
Informally, $H_0^m(\Omega)$ 
is the set of functions 
belonging to $H^m(\Omega)$
that can be approximated by a sequence
of functions $\phi_n \in C_0^\infty(\Omega)$.
This also implies that if a function
$g \in H_0^m(\Omega)$,
then $g(x) = 0$ for all $x \in \partial \Omega$.
This space (particularly with $m=1$)
is often useful when analyzing elliptic PDEs 
with Dirichlet boundary conditions.

% First, note that the 
% unique solution $u^\star$ to the PDE 
% in Definition~\ref{eq:main_pde}, 
% satisfies the following,
% \begin{equation}
%     \label{eq:weak_solution_formulation}
%     \int_\Omega \left(A \nabla u \cdot \nabla v + cuv\right)\; dx = \int_\Omega fv \; dx, \qquad \forall v \in \h
% \end{equation}
% The proof for this is standard and is provided in the appendix (Proposition~\ref{proposition:weak_solution}) for completeness. 
% Further, by re-writing 
% $\int_\Omega \left(A \nabla u \cdot \nabla v + cuv\right)\; dx = \langle Lu, v\rangle$ 
% where $L$ is a functional operator,
% and 
% $\int_\Omega fv \; dx = (f, v)$ we can rewrite (\ref{eq:weak_solution_formulation}) in a more abstract form,
% \begin{equation}
%     \label{eq:variational_formulation}
%     \langle Lu, v\rangle =  (f, v), \qquad v \in \h.
% \end{equation}
% \input{sections/Proposition1/proposition1}
\begin{definition}[Weak Solution]
    \label{def:weak_solution}
    % Given $\Omega$, the function $u: \Omega \to \R$
    % solves the PDE in Definition~\ref{eq:main_pde} a weak sense 
    Given the PDE in Definition~\ref{eq:main_pde},
    % with $f \in \l2$,
    if $f \in \l2$,
    then a function $u : \Omega \to \R$ 
    solves the PDE in a weak sense
    if $u \in \h$ and for all $v \in \h$, we have
    \begin{equation}
        \label{eq:variational_formulation}
        \int_\Omega \left(A \nabla u \cdot \nabla v + cuv\right) dx = \int_\Omega fv dx
    \end{equation}
\end{definition}
The left hand side of 
\eqref{eq:variational_formulation}
is also equal to $\langle Lu, v\rangle_{\l2}$ 
for all $u, v \in \h$ 
(see Lemma~\ref{proposition:weak_solution}), 
whereas, following the definition of the $\l2$ norm,
the right side 
% is nothing but
is simply $\langle f, v\rangle_{\l2}$.
Having introduced these preliminaries, 
we now introduce some important 
% concepts regarding linear PDEs
facts about linear PDEs
that feature prominently in our analysis.
% concepts %and formulations
% and formulations for a linear PDE,
% which we will use extensively in developing our analysis.

\begin{proposition}
    \label{p:maingd}
    For the PDE in Definition~\ref{eq:main_pde}, 
    if $f \in \l2$
    % following holds,
    the following hold:
    \begin{enumerate}
        % \item The PDE in Definition~\ref{eq:main_pde} has a unique solution $u^\star$.
        \item The solution
        to Equation~\eqref{eq:variational_formulation}
        exists and is unique.
        \item The weak solution is also the unique solution 
        of the following minimization problem:
            \begin{equation}
                \label{eq:minimization_problem}
                u^\star = \argmin_{v \in \h} J(v) := 
                \argmin_{v \in \h} \left\{\frac{1}{2}\langle Lv, v \rangle_{\l2} - \langle f, v\rangle_{\l2}\right\}.
            \end{equation}
    \end{enumerate}
\end{proposition}
% The above 
This proposition is standard
(we include a proof in the Appendix, Section~\ref{section:Proof_of_proposition_1} for completeness) and  
states that there exists a unique solution 
to the PDE (referred to as $u^\star$),
which is also the solution we get 
from the variational formulation 
in ~\eqref{eq:minimization_problem}.
% We therefore 
In this work, we introduce 
a sequence of functions 
that minimizes the loss 
in the variational formulation.

\begin{definition}[Eigenvalues and Eigenfunctions,~\citet{evans1998partial}]
    \label{definition:eigenvalue_eigenfunctions}
    Given an operator $L$, the tuples $(\lambda, \varphi)_{i=1}^\infty$,
    where $\lambda_i \in \R$ 
    and $\varphi_i \in \h$
    are
    (eigenvalue, eigenfunction) pairs 
    that 
    satisfy %the 
    $L\varphi = \lambda \varphi$, for all $x \in \Omega$.
    % following:
    % \begin{align*}
    %         L\varphi = \lambda \varphi \qquad \forall x \in \Omega. 
    % \end{align*}
    Since $\varphi \in \h$,
    we know that $\varphi_{|\partial \Omega} = 0$.
    The eigenvalue can be written as
    \begin{equation}
        \label{eq:eigenvalue_definition}
        \lambda_i = \inf_{u \in X_i} \frac{\langle L u,u \rangle_{\l2}}{\|u\|^2_{\l2}},
    \end{equation}
    where 
    $X_i := \mbox{span}\{\varphi_1, \ldots, \varphi_{i}\}^{\perp} 
    = \{u \in \h: \langle u, \varphi_j\rangle_{\l2} = 0\; \forall j \in \{1, \cdots, i\}\}$
    and $0 < \lambda_1 \leq \lambda_2 \leq \cdots $. Furthermore, we define by $\Phi_k$ the span of the first $k$
eigenfunctions of $L$, i.e., $\Phi_k := \spn\{\varphi_1, \cdots, \varphi_k\}$.

\end{definition}
% Given that the operator $L$ 
We note that since the operator $L$ 
% is in fact self-adjoint and elliptic,
is self-adjoint and elliptic (in particular, $L^{-1}$ is compact),
the eigenvalues are real and countable. % in number.
Moreover, the eigenfunctions 
form an orthonormal basis of $\h$ 
(see \citet{evans1998partial}, Section 6.5).

\section{Main Result}
\label{section:main_result}

Before stating our results, 
we provide the formal assumptions 
on the PDEs of interest:

\paragraph{Assumptions:}
\begin{enumerate}[label=(\roman*)]
    \item \textbf{Smoothness}: 
        We assume that $\partial \Omega \in C^\infty$.
        We also assume that the coefficient 
        $A\in \Omega \rightarrow \R^{d \times d}$ 
        is a symmetric matrix-valued function, i.e., 
        $A = (a_{ij}(x))$ and 
        $a_{ij}(x) \in L^\infty(\Omega)$ 
        % is infinitely differentiable 
        for all $i, j \in [d]$
        % Furthermore, 
        and 
        the function $c \in L^\infty(\Omega)$
        % is also infinitely differentiable 
        and $c(x) \geq \zeta > 0$ 
        for all $x \in \Omega$.
        Furthermore, we assume that $a_{ij}, c \in C^\infty$.
        We define a constant $$C := \Cvalue.$$
        Further, the function $f \in \l2$
        is also in $C^\infty$ 
        and 
        the projection of $f$ onto $\Keigenfunctions$ which we denote
        $\fkeigen$
        satisfies %$\|f - \fkeigen\|_{\l2} \leq \epsilon_{\spn}$ 
        %and 
        % there exists a $D < \infty$ such that 
        for any multi-index $\alpha$: 
        $\|\partial^\alpha f - \partial^\alpha \fkeigen\|_{\l2} \leq \epsilon_{\spn}$.
        \footnote{Since $\partial \Omega \in C^\infty$ and the functions $a_{ij}, c$ and $f$ are all in $C^\infty$,
        it follows from \citet{nirenberg1955remarks} (Theorem, Section 5)
        the eigenfuntions of $L$ are also $C^\infty$. 
        Hence, the function $\fkeigen$ is in $C^\infty$ as well.
        }
    \item \textbf{Ellipticity}: 
        There exist constants 
        $M \geq m > 0$ such that, 
        for all $x \in \Omega$ and $\xi \in \R^d$,
        $$ m\|\xi\|^2 \leq \sum_{i,j=1}^d a_{ij}(x)\xi_i\xi_j \leq M \|\xi\|^2.$$
    \item \textbf{Neural network approximability}:
          There exist neural networks $\ta$ and $\tc$ 
          with $N_A, N_c \in \N$ parameters, respectively,
          that approximate the functions $A$ and $c$, i.e.,
          $ \|A - \ta\|_{L^\infty(\Omega)} \leq \epsilon_A$
          and $ \|c - \tc \|_{L^\infty(\Omega)} \leq \epsilon_c$,
          for small $\epsilon_A, \epsilon_c \geq 0$.
          We assume that for all $u \in \h$
          the operator $\tl$ defined as,
          \begin{equation}
            \label{eq:tilde_L}
            \tl u = -\divergence(\ta \nabla u) + \tc u .
          \end{equation}
          is elliptic with 
          $(\tilde{\lambda}_i, \tilde{\varphi}_i)_{i=1}^\infty$
          (eigenvalue, eigenfunction) pairs.
          We also assume that there exists 
          a neural network $\fnn \in C^{\infty}$ 
          with $\fnumparams \in N$ parameters
          such that for any multi-index $\alpha$,
         $\|\partial^\alpha f - \partial^\alpha \fnn \|_{\l2} \leq \epsilon_{\nn}$.
          By $\Sigma$, we denote the set 
          of all (infinitely differentiable) 
          activation functions
          used by networks $\ta$, $\tc$, and $\fnn$.
          By $\Sigma'$, we denote the set 
          that contains all the $n$-th order derivatives 
          of the activation functions 
          in $\Sigma$, $\forall n \in \mathbb{N}_0$
\end{enumerate}

Intuitively, ellipticity of $L$ 
in a linear PDE $Lu = f$ 
is analogous to positive definiteness 
of a matrix $Q \in \mathbb{R}^d$ 
in a linear equation $Qx = k$, 
where $x,k \in \R^d$. 

In (iii), we assume that
the coefficients $A$ and $c$, 
and the function $f$
can be approximated 
by neural networks. 
% While true for any
While this is true for any smooth functions
given sufficiently large $N_A, N_c, N_f$,
our results are most interesting 
when these quantities are small 
(e.g. subexponential in the input dimension $d$). 
For many PDEs used in practice,
approximating the coefficients 
using small neural networks 
is straightforward. 
For example, in heat diffusion 
(whose equilibrium is defined 
by a linear elliptic PDE)
$A(x)$ defines the conductivity 
of the material at point $x$.
If the conductivity is constant,
then the coefficients can be written
as neural networks with $O(1)$ parameters. 

The part of assumption (i) 
that stipulates that $f$ is close to $\fkeigen$ 
can be thought of as 
a smoothness condition on $f$.
For instance, if $L = -\Delta$
(the Laplacian operator), 
the Dirichlet form satisfies 
$\frac{\langle L u,u \rangle_{\l2}}{\|u\|^2_{\l2}} = \frac{\|\nabla u\|_{\l2}}{\|u\|_{\l2}}$, 
so eigenfunctions 
corresponding to higher eigenvalues 
tend to exhibit 
a higher degree of spikiness.
The reader can also think 
of the eigenfunctions 
corresponding to larger $k$ 
as Fourier basis functions 
corresponding to higher frequencies.

% corresponding 
% in assumption iii) k is not talking about the top vs lower eigenfunctions but to kth order derivatives
% confusing, why do we need to say "k" at all? can swap it out for some other letter yes. 
% 
%to larger $k$ 

Finally, in (i) and (iii), 
while the requirement 
that the function pairs
($f$, $f_{\nn}$)
and ($f$, $f_{\spn}$)
are close not only in their values, 
but their derivatives as well
is a matter of analytical convenience,
our key results do not necessarily
depend on this precise assumption.
Alternatively, we could replace this assumption
with similar (but incomparable) conditions: 
e.g., we can also assume closeness of the values 
and a rapid decay of the $L^2$ norms of the derivatives. 
We require control over the derivatives 
because our method's 
gradient descent iterations 
involve repeatedly applying 
the operator $L$ to $f$---which results 
in progressively higher derivatives.

We can now formally state our main result:
\begin{theorem}[Main Theorem]
    \label{thm:main_result}
    Consider a linear elliptic PDE satisfying Assumptions \numassumptions,
    and let $u^\star \in \h$ denote its unique solution.
    If there exists a neural network $u_0 \in H_0^1(\Omega)$ with $N_0$ parameters,
    such that 
    $\|u^\star - u_0\|_{\l2} \leq R$, for some $R < \infty$, 
    then for every $T \in \mathbb{N}$ such that $T  \leq \frac{1}{20 \min(\lambda_k, 1) \delta}$, 
    there exists a neural network $u_T$ 
    with size 
    $$O\left(d^{2T}\left(N_0 + N_A\right) 
    + T(\fnumparams + N_c)\right)$$
    such that 
     $\|u^\star - u_T\|_{\l2} \leq \epsilon + \tilde{\epsilon}$
    where 
    \begin{align*}
        \epsilon &:= \left(\invkappa\right)^T R,\\ 
        \tilde{\epsilon} &:=
    \frac{\epsilon_{\spn}}{\lambda_1} 
    + \frac{\delta}{\lambda_1}\errordaviskahan
    + \delta \|u^\star\|_{\l2}
    % + (\max\{1, T^2 C \eta\})^{T} \left(\epsilon_{\spn} + \epsilon_{nn} +  \frac{2^{3/2} \delta \|f\|_{\l2}}{\gamma - \delta}\right),
    + (\max\{1, T^2 C \eta\})^{T} \left(\epsilon_{\spn} + \epsilon_{nn} 
    + \constantnumber\left(1 + \constantdeltagamma\right)\lambda_k^T\|f\|_{\l2}
    % + \errorspanupdate
    \right),
    \end{align*}
    and 
    % and
    %$\kappa := \kappavalue$, 
    % $C_d = (2d^2 + 1)C$,
    $\eta := \learningrate$,
    % $T$ is $O\left(\frac{\log \left(R/\epsilon\right)}{\log \kappa}\right)$, 
    $\delta := \deltavalue$.
    %such that $\delta \leq \deltarange$
    %and $\alpha$ is a multi-index.
    Furthermore, the activation functions used in $u_T$ belong to the set
    $\Sigma \cup \Sigma' \cup \{\rho\}$ 
    where 
    $\rho(y) = y^2$ for all $y \in \R$ is the square activation function.
\end{theorem}

% Before we prove the theorem (Section 
% \ref{section:proof_of_main_result}),
% We briefly summarize the result
% and present a few remarks and prove the t
% 
% The above theorem statement shows that
This theorem shows that given an 
initial neural network $u_0 \in \h$ 
containing $N_0$ parameters,
we can recover a neural network 
that is $\epsilon$ close 
to the unique solution $u^\star$.
The number of parameters in $u_\epsilon$
depend on how close the initial estimate $u_0$ is 
to the solution $u^\star$, and $N_0$. 
This results in a trade-off, 
where better approximations 
may
require more parameters,
compared to a poorer approximation with fewer parameters.

Note that $\epsilon \to 0$ as $T \to \infty$,
while $\tilde{\epsilon}$ is a ``bias'' error term that does not go to $0$ as $T \to \infty$. The 
first three terms in the 
expression for $\tilde{\epsilon}$ 
result from bounding the difference between the solutions to the equations 
$L u = f$ and $\tilde{L} u = \fkeigen$, whereas the third term is due to difference between $f$ and $f_{\nn}$ and the fact that our proof involves simulating the gradient descent updates with neural networks.
Further, if the constant $\zeta$ is equal to $0$ then the error term $\epsilon_c$ will also be $0$,
in which case the term $\delta$ will equal $\epsilon_A/m$.

%Further, the $O(T^{2T}(\D))$ dependence in the error term 
% comes from the part of 
% the neural network $\fnn$ that lies outside the span 
% of the first k eigenfunctions of the operator $L$.
%is due to the neural network $\fnn$ 
%lying outside the span 
%of the first $k$ eigenfunctions
%of the operator $L$.
%Our sequence involves the repeated application 
%of $L$, at each step the number of gradients
%calculate increases and therefore
%the bounds depends upon $\D$, 
%the upper bound on the 
%difference of
%gradient terms.
%Therefore, we would ideally want 
%a neural network approximation $\fnn$
%that 
%approximately
%lies within $\Keigenfunctions$.

% For our proof we introduce
% an iterative sequence with roughly $O\left(\frac{\log\left(R/\epsilon\right)}{\log\left(\kappa\right)}\right)$
% functions to reach an $\epsilon$-approximation to the solution $u^\star$.
% \hl{The parameter $\kappa$ depends upon $\lambda_k$ and $\lambda_1$, and we will 
% later show that it acts as analogous to the condition number of a matrix.}
% Therefore, if $\lambda_k \gg \lambda_1$ then the number of parameters in $u_\epsilon$ increases.

% Further, the activation function $\sigma$ ensures that 
% the final network $u_\epsilon$ satisfies the Dirichlet boundary condition.

The fact that $\epsilon := \left(\invkappa\right)^T R$ %$T = O\left(\frac{\log\left(R/\epsilon\right)}{\log\left(\kappa\right)}\right)$
comes from the fact that we are simulating $T$ steps of a gradient descent-like procedure on a strongly convex loss. % to reach an $\epsilon$-approximate optimum. 
%to the solution $u^\star$.
%The parameter $\kappa$ depends 
The parameters $\tilde{\lambda}_k$ and $\tilde{\lambda}_1$ can be thought of as the effective Lipschitz and strong-convexity constants of the loss. 
% 
%and acts 
% as analogous 
%analogously
%to the condition number of a matrix\footnote{
%If $\lambda_k \gg \lambda_1$,
% then it implies that
%then
%$\kappa \approx 1 + \lambda_1/\lambda_k$ 
%and hence $1/\log(\kappa) \approx \lambda_k / \lambda_1$.
%}.
Finally, 
to give a sense 
of what $R$ looks like,
we show in Lemma~\ref{corollary:initialization_with_u0}
% (see Section~\ref{section:appendix_main_result} in the Appendix)
that if $u_0$ is initialized to be identically zero then 
$R \leq \frac{\|f\|_{\l2}}{\lambda_1}$.

\begin{lemma}
    \label{corollary:initialization_with_u0}
    If $u_0=0$, then $R \leq \frac{\|f\|_{\l2}}{\lambda_1}$.
    \iffalse 
    $u_0 \in \h$ (as it satisfies the boundary condition),
    then the number of parameters required in $u_\epsilon$ is bounded by
    $$O\left(d^{2\frac{\log\left(\frac{\|f\|_{\l2}}{\lambda_1\epsilon}\right)}{\log{\kappa}}}\left(N_0 + N_A\right) 
    + \frac{\log\left(\frac{\|f\|_{\l2}}{\lambda_1\epsilon}\right)}{\log{\kappa}}(\fnumparams + N_c)\right)$$
    \fi
\end{lemma}
\begin{proof}
    Given that $u_0$ is identically $0$,
    the value of $R$ in 
    Theorem~\ref{thm:main_result} equals $\|u^\star - u_0\|_{\l2} = \|u^\star\|_{\l2}$
    Using the inequality in (\ref{eq:operator_inequality}), we have,
    \begin{align*}
        \|u^\star\|_{\l2}^2 &\leq \frac{\langle Lu^\star, u^\star \rangle}{\lambda_1}\\
        &\leq \frac{1}{\lambda_1}\langle f, u^\star\rangle_{\l2} \\
        &\leq \frac{1}{\lambda_1}\|f\|_{\l2}\|u^\star\|_{\l2}\\
    \implies  \|u^\star\|_{\l2} &\leq \frac{1}{\lambda_1}\|f\|_{\l2}
    \qedhere
    \end{align*}
\end{proof}

We make few remarks about the theorem statement: 

\begin{remark}
    \label{remark:unique_solution_remark}
    While we state our convergence results in $\l2$ norm,
    our proof works for the $\h$ norm as well.
    This is because in the space defined by the top-k eigenfunctions of the operator $L$, 
    $\l2$ and $\h$ norm are equivalent (shown in Proposition~\ref{proposition:equivalence_L2_H01}).
    Further, note that even though we have assumed that $u^\star \in \h$ is the unique solution
    of (\ref{eq:variational_formulation}) from the boundary regularity condition, 
    we have that $u^\star \in H^2(\Omega)$
    (see \citet{evans1998partial}, Chapter 6, Section 6.3).
    This ensures that the solution $u^\star$ is twice differentiable as well.
\end{remark}

\begin{remark}
    \label{remark:laplacian_remark}
    To get a sense of the scale of $\lambda_1$ and $\lambda_k$, when $L = -\Delta$ (the Laplacian operator), the eigenvalue $\lambda_1 = \inf_{u \in \h} \frac{\|\nabla u\|_{\l2}}{\|u\|_{\l2}} = \frac{1}{C_p}$, 
    where $C_p$ is the Poincar\'e constant (see Theorem~\ref{thm:poincare_inequality} in Appendix). For geometrically well-behaved sets $\Omega$ (e.g. convex sets with a strongly convex boundary, like a sphere), $C_p$ is even dimension-independent. 
    Further from the Weyl's law  operator (\cite{evans1998partial}, Section 6.5) we have 
    $$\lim_{k \to \infty} \frac{\lambda_k^{d/2}}{k} = \frac{(2\pi)^d}{\mathrm{vol}(\Omega)\alpha(d)}$$
    where $\alpha(d)$ is the volume of a unit ball in $d$ dimensions. So, if $\mathrm{vol}(\Omega) \geq 1/\alpha(d)$, $\lambda_k$ grows as $O(k^{2/d})$, which is a constant so long as $\log k \ll d$.  
\end{remark}

\section{Proof of Main Result}
\label{section:proof_of_main_result}

First, we provide some intuition behind the proof, 
via an analogy between 
a uniformly elliptic operator 
and a positive definite matrix in linear algebra. 
We can think of finding the solution 
to the equation $Lu = f$ 
for an elliptic $L$ 
as analogous to finding the solution 
to the linear system of equations $Qx = k$, 
where $Q$ is a $d \times d$ 
positive definite matrix,
and $x$ and $k$ are $d$-dimensional vectors.
One way to solve such a linear system 
is by minimizing the strongly convex function 
$\|Qx - b\|^2$ using gradient descent. 
Since the objective is strongly convex, 
after $O(\log(1/\epsilon))$ gradient steps, 
we reach an $\epsilon$-optimal point in an $l_2$ sense. 

Our proof uses a similar strategy. 
First, we show that for the operator $L$, 
we can define a sequence of functions that converge 
to an $\epsilon$-optimal function approximation 
(in this case in the $\l2$ norm) 
after $O(\log(1/\epsilon)$ steps---similar 
to the rate of convergence 
for strongly convex functions.  
Next, we inductively show
that each iterate in the sequence 
can be approximated by a small neural network. 
More precisely, 
we show that given a bound 
on the size of the $t$-th iterate $u_t$,
we can, in turn, upper bound 
the size of the $(t+1)$-th iterate $u_{t+1}$
because the update transforming $u_t$ to $u_{t+1}$ 
can be simulated by a small neural network (Lemma~\ref{lemma:recursion_lemma}).
These iterations look roughly
like $u_{t+1} \leftarrow u_t - \eta (L u_t - f)$, 
and we use a ``backpropagation'' lemma 
(Lemma \ref{thm:backpropagation}) 
which bounds the size of 
the derivative of a neural network.

\subsection{Defining a Convergent Sequence}
The rough idea is to perform 
gradient descent in $\l2$
\citep{neuberger2009sobolev, farago2001gradient, farago2002numerical}
to define a convergent sequence
whose iterates converge 
to $u^\star$ in $\l2$ norm
(and following Remark~\ref{remark:unique_solution_remark}, in $\h$ as well).
However, there are two obstacles 
to defining the iterates 
as simply $u_{t+1} \leftarrow u_t - \eta (L u_t - f)$:
(1) $L$ is unbounded---so the standard way 
of choosing a step size for gradient descent 
(roughly the ratio of the minimum 
and maximum eigenvalues of $L$)
would imply choosing a step size $\eta = 0$, 
and 
(2) $L$ does not necessarily 
preserve the boundary conditions,
so if we start with $u_t \in \h$, 
it may be that $L u_t - f$
does not even lie in $\h$.

We resolve both issues 
by restricting the updates 
to the span of the first 
$k$ eigenfunctions of $L$.
More concretely, as shown 
in Lemma~\ref{lemma:bounded_operator_norm}, 
if a function $u$ in $\Keigenfunctions$, 
then the function $Lu$ 
will also lie in $\Keigenfunctions$.
We also show that within 
the span of the first $k$ eigenfunctions, 
$L$ is bounded (with maximum eigenvalue $\lambda_k$),
and can therefore be viewed as an operator 
from $\Keigenfunctions$ to $\Keigenfunctions$.
Further, we use $\fkeigen$
instead of $f$ in our updates, 
which now have the form 
$u_{t+1} \leftarrow u_t - \eta(L u_t - \fkeigen)$.
Since $\fkeigen$ belongs to $\Keigenfunctions$, 
for a $u_t$ in $\Keigenfunctions$
the next iterate $u_{t+1}$ 
will now remain in $\Keigenfunctions$.
Continuing the matrix analogy, 
we can choose the usual step size 
of $\eta = \learningratenotilde$. 
Precisely, we show: 
\begin{lemma}
    \label{lemma:bounded_operator_norm}
    Let $L$ be an elliptic operator. Then, 
    for all $v \in \Keigenfunctionsmath$  
    it holds:
    \begin{enumerate}
        \item $Lv \in \Keigenfunctions$.
        \item 
             $   \label{eq:operator_inequality}
                    \lambda_1 \|v\|_{\l2} \leq \langle Lv, v \rangle_{\l2}  \leq \lambda_k \|v\|_{\l2}
            $
        \item  
                $   \label{eq:proof_for_upper_bounding_operator_norm_k_eigenfunctions}
                   \left\|\left(I - \frac{2}{\lambda_k + \lambda_k}L\right)u\right\|_{\l2} \leq  \frac{\lambda_k - \lambda_1}{\lambda_k +\lambda_1}\|u\|_{\l2}
               $
    \end{enumerate}
\end{lemma}
\begin{proof}
    Writing $u \in \Keigenfunctions$ as $u = \sum_i d_i \varphi_i$
    where $d_i = \langle u, \varphi_i\rangle_{\l2}$, we have
    $ L u = \sum_{i=1}^k \lambda_i d_i \varphi_i$. Therefore $L u \in \Keigenfunctionsapprox$ and $L u$ lies in  $\h$, proving (1.).

    Since $v \in \Keigenfunctionsmath$, we use the definition of eigenvalues in (\ref{eq:eigenvalue_definition}) to get,
    \begin{align*}
        \frac{\langle Lv, v\rangle_{\l2}}{\|v\|_{\l2}} \leq \sup_{v} \frac{\langle Lv, v\rangle_{\l2}}{\|v\|_{\l2}} =  \lambda_k \\
        \implies \langle Lv, v \rangle_{\l2} \leq \lambda_k \|v\|^2_{\l2}
    \end{align*}
    and similarly 
    \begin{align*}
        \frac{\langle Lv, v\rangle_{\l2}}{\|v\|_{\l2}} \geq \inf_{v} \frac{\langle Lv, v\rangle_{\l2}}{\|v\|_{\l2}} = \lambda_1 \\
        \implies \langle Lv, v \rangle_{\l2} \geq \lambda_1 \|v\|^2_{\l2}
    \end{align*}

    In order to prove (2.) let us first
    denote $\bar{L} := \left(I - \frac{2}{\lambda_k + \lambda_1}L\right)$. 
    Note if $\varphi$ is an eigenfunction of $L$ with corresponding eigenvalue $\lambda$, 
    it is also an eigenfunction of $\bar{L}$ with corresponding eigenvalue $\frac{\lambda_k+ \lambda_1 - 2\lambda}{\lambda_k+\lambda_1}$. 

    Hence, writing $u \in \Keigenfunctionsmath$ as 
    $u = \sum_{i=1}^k d_i \varphi_i$,
    where $d_i = \langle u, \varphi_i\rangle$, we have  
    \begin{equation}
        \|\bar{L}u\|_{\l2}^2 = \left\| \sum_{i=1}^k \frac{\lambda_k+ \lambda_1 - 2\lambda_i}{\lambda_k+\lambda_1} d_i \varphi_i\right\|_{\l2}^2 
        \leq \max_{i \in k} \left(\frac{\lambda_k+ \lambda_1 - 2\lambda_i}{\lambda_k+\lambda_1}\right)^2 \left\|\sum_{i=1}^k d_i \varphi_i \right\|_{\l2}^2 
        %= |\tilde{\lambda}_k| \sum_{i=1}^k d^2_i  
        %= |\tilde{\lambda}_k| \sum_{i=1}^\infty d^2_i  
        %= \frac{\lambda_k - \lambda_1}{\lambda_k + \lambda_1} \|u\|_{\l2}
    \label{eq:ineqL}
    \end{equation}
    By the orthogonality of $\{\varphi_i\}_{i=1}^k$, we have 
    $$\left\|\sum_{i=1}^k d_i \varphi_i \right\|_{\l2}^2 = \sum_{i=1}^k d_i^2 = \|u\|_{\l2}^2$$
    Since $\lambda_1 \leq \lambda_2 \dots \leq \lambda_k$, we have $ \lambda_k + \lambda_1 - 2 \lambda_i \geq \lambda_1 - \lambda_k $
    and $ \lambda_k + \lambda_1 - 2 \lambda_i \leq \lambda_k - \lambda_1 $, so $|\lambda_k + \lambda_1 - 2 \lambda_i| \leq \lambda_k - \lambda_1$. This implies $\max_{i \in k}\left(\frac{\lambda_k+ \lambda_1 - 2\lambda_i}{\lambda_k+\lambda_1}\right)^2 \leq \left(\frac{\lambda_1-\lambda_k}{\lambda_1+\lambda_k}\right)^2$. Plugging this back in \eqref{eq:ineqL}, we get the claim we wanted.
\end{proof}

In fact, we will use 
a slight variant of the updates
and instead set 
$u_{t+1} \leftarrow u_t - \eta(\tl u - \fkeigenapprox)$
as the iterates of the convergent sequence,  where $\fkeigenapprox$
    is the projections of $f$ onto $\Keigenfunctionsapprox$.
This sequence satisfies two important properties: (1) The convergence point 
of the sequence and $u^\star$, 
the solution to the original PDE, 
are not too far from each other; (2) The sequence of functions converges exponentially fast. In Section~\ref{section:network_approximation},
we will see that updates defined thusly %in this form
will be more convenient 
to simulate via a neural network. 

The first property is formalized as follows: 

\begin{lemma}
    \label{lemma:error_between_ustar_uapproxsolution}
    Assume that $\uapproxsolution$ 
    is the solution to the PDE
    $\tl u = \fkeigenapprox$,
    where $\fkeigenapprox: \h \to \R$
    is the projections of $f$ onto $\Keigenfunctionsapprox$. 
    Given Assumptions \numassumptions, 
    we have
    $\|u^\star - \uapproxsolution\|_{\l2} \leq \epsilon$, 
    such that $\epsilon = \errorspan$, 
    where $\gamma = \gammavalue$ 
    and $\delta = \deltavalue$.
\end{lemma}

The proof for Lemma~\ref{lemma:error_between_ustar_uapproxsolution}
is provided in the Appendix
(Section \ref{section:lemma_error_between_ustar_uapproxsolution}).
Each of the three terms in the final error 
captures different sources of perturbation:
the first term comes from 
approximating $f$ by $f_{\spn}$;
the second term comes from applying Davis-Kahan~\citep{davis1970rotation} 
to bound the ``misalignment'' 
between the eigenspaces $\Keigenfunctions$ 
and $\Keigenfunctionsapprox$ 
(hence, the appearance of the eigengap 
between the $k$ and $(k+1)$-st eigenvalue of $L^{-1}$); 
the third term is a type of ``relative'' error
bounding the difference 
between the solutions to the PDEs 
$L u = \tilde{f}_{\spn}$ and 
$\tilde{L} u = \tilde{f}_{\spn}$.  

\iffalse
Here, the operator $\tl$ 
(defined in \eqref{eq:tilde_L})
has the neural network approximations 
of $A$ and $c$ as its coefficients, 
and $\fkeigenapprox$ 
is a projection of the function $f$ on the
first $k$ eigenfunctions 
of $\tl$ (denoted by $\Keigenfunctionsapprox$).
We will show that
$\|\fkeigen - \fkeigenapprox\|_{\l2}$ 
is small in Lemma~\ref{lemma:fkeigen_fkeigenapprox_davis_kahan} below.
In Section~\ref{section:network_approximation},
we will see that updates defined thusly %in this form
will be more convenient 
to simulate via a neural network.
\fi

The ``misalignment'' term can be characterized through the following lemma: 
\begin{lemma}[Bounding distance between $\fkeigen$ and $\fkeigenapprox$]
    \label{lemma:fkeigen_fkeigenapprox_davis_kahan}
    Given Assumptions \numassumptions  
    and denoting the
    projection of $f$ onto $\Keigenfunctionsapprox$ by $\fkeigenapprox$ we have:
    \begin{equation}
        \label{eq:error_fkeigen_fkeigenapprox}
        \|\fkeigen - \fkeigenapprox\|_{\l2} \leq \errorspanstar
    \end{equation}
    where $\delta = \deltavalue$.
\end{lemma}
\begin{proof}
    Let us write $\fkeigen = \sum_{i=1}^k  f_i \varphi_i$ where $f_i = \langle \errorf, \varphi_i\rangle_{\l2}$.
    Further, we can define a function 
    $\fkeigenapprox \in \Keigenfunctionsapprox$
    such that $\fkeigenapprox = \sum_{i=1}^k \tilde{f}_i \tvarphi_i$ such that
    $\tilde{f}_i = \langle \errorf, \tvarphi_i\rangle_{\l2}$. 
    
    If $P_k g := \sum_{i=1}^k \langle g, \varphi_i\rangle_{\l2} \varphi_i$
    and 
    $\tilde{P}_k g := \sum_{i=1}^k \langle g, \tvarphi_i\rangle_{\l2} \tvarphi_i $
    denote the projection of a function 
    $g$ onto $\Keigenfunctions$ and $\Keigenfunctionsapprox$,
    %Using the definition of projection operators 
    % $P_k g = \sum_{i=1}^k \langle g, \varphi_i \rangle_{\l2} \varphi_i$ and 
    % $\tilde{P}_k g = \sum_{i=1}^k \langle g, \tvarphi_i \rangle_{\l2} \tvarphi_i$,
    %$P_k g$ and $\tilde{P}_k g$,
    from Lemma~\ref{lemma:perturbation_bound}, we have:
    \begin{align*}
        \|\fkeigen - \fkeigenapprox\|_{\l2}
            &= \left\|\sum_{i=1}^k \langle \errorf, \varphi_i\rangle_{\l2} \varphi_i 
                - \langle \errorf, \tvarphi_i\rangle_{\l2} \tvarphi_i\right\|_{\l2}\\
            &= \left\| P_k \errorf - \tilde{P}_k \errorf\right\|_{\l2}\\
            &\leq \|P_k - \tilde{P}_k\| \|\errorf\|_{\l2} \\
            &\leq \constantdeltagamma\|\errorf\|_{\l2}
    \end{align*}
    where
    $\gamma = \frac{1}{\lambda_k} - \frac{1}{\lambda_{k+1}}$,
    and $\delta = \deltavalue$.
\end{proof}

The main technical tool for bounding the difference between the operators $L$ and $\tilde{L}$ can be formalized through the lemma below. Note, the ``relative'' nature of the perturbation is because $L$ and $\tilde{L}$ are not bounded operators.
\begin{lemma}[Relative operator perturbation bound]
    \label{lemma:L_tl_inner_product_close}
    Consider the operator $\tl$ defined in \eqref{eq:tilde_L}, 
    then for all $u \in \h$
    we have the following:
    \begin{enumerate}
        \item $\langle (\tl - L)u, u\rangle \leq \delta \langle Lu, u\rangle$
        \item $\langle (L^{-1} \tl - I)u, u \rangle_{\l2} \leq  \delta \|u\|_{\l2}^2$
    \end{enumerate}
    where $\delta = \deltavalue$.
\end{lemma}
\begin{proof}
  \begin{align*}
        \langle (\tl - L)u, u\rangle 
        &= \int_\Omega \left((\ta - A)\nabla u\cdot \nabla u + (\tc - c)u^2\right) dx\\
        &\leq \left(\max_{ij}\|\ta_{ij} - A_{ij}\|_{\linf}\right)\|\nabla u\|^2_{\l2} + \|\tc - c\|_{\linf}\|u\|^2_{\l2}\\
        % &\leq \max\{\epsilon_A, \epsilon_c\}\left(\|\nabla u\|^2_{\l2} + \|u\|^2_{\l2}\right)
        &\leq \epsilon_A \|\nabla u\|_{\l2}^2 
        + \epsilon_c \|u\|^2_{\l2}
        \numberthis \label{eq:final_LminusLtilde_innerp}
  \end{align*} 
  Further, note that
  \begin{align*}
      \langle Lu, u\rangle &= \int_{\Omega} A \nabla u \cdot \nabla u + cu^2 dx \\
    %   &\geq \min\{m, \zeta\} \left(\|\nabla u\|^2_{\l2} + \|u\|^2_{\l2}\right)
      &\geq 
      m\|\nabla u\|^2_{\l2} + \zeta \|u\|_{\l2}^2
    %   m\|\nabla u\|^2_{\l2} 
      \numberthis \label{eq:final_L_lbd}
  \end{align*}
  Using the inequality 
  $\frac{a + b}{c + d} \geq \min\{\frac{a}{c}, \frac{b}{d}\}$
  from \eqref{eq:final_LminusLtilde_innerp} 
  and \eqref{eq:final_L_lbd}, we have
  \begin{align}
      \frac{m\|\nabla u\|^2_{\l2} + \zeta \|u\|_{\l2}^2}
      {\epsilon_A\|\nabla u\|_{\l2}^2 + \epsilon_c \|u\|_{\l2}^2}
      \geq \min\left\{\frac{m}{\epsilon_A}, \frac{\zeta}{\epsilon_c}\right\}
  \end{align}
%   \begin{align}
%       \frac{m\|\nabla u\|^2_{\l2}}
%       {\epsilon_A\|\nabla u\|_{\l2}^2 + \epsilon_c \|u\|_{\l2}^2}
%       \geq \left\{\frac{m}{\epsilon_A}\right\}
%   \end{align}
  
  Hence this implies that 
  $$\langle (\tilde{L} - L)u, u\rangle \leq \delta \langle Lu, u\rangle$$
    where $\delta = \deltavalue$ 
    proving part (1.).
% \Anote{this remark doesn't belong here} Note that if $\zeta = 0$ it also implies that $\epsilon_c = 0$ and in that case $\delta = \frac{1}{m/\epsilon_A}$.

    Further, for part (2.)
    we have for all $u \in \h$,
    \begin{align*}
        &\langle (\tl - L)u, u \rangle_{\l2} \leq \delta \langle Lu, u\rangle_{\l2}\\
        \implies &\langle (\tl L^{-1} - I)Lu, u \rangle_{\l2} \leq \delta \langle Lu, u\rangle_{\l2}\\
        \implies &\langle (\tl L^{-1} - I)v, u \rangle_{\l2} \leq \delta \langle v, u\rangle_{\l2}\\
        \implies &\langle (\tl L^{-1})v, u \rangle_{\l2} \leq (1 + \delta) \langle v, u\rangle_{\l2}
        \numberthis \label{eq:LtildeL_inv_final}
    \end{align*}
    where $v = Lu$. Therefore using \eqref{eq:LtildeL_inv_final} the following holds for all $u \in \h$, 
    \begin{align*}
        &\langle (\tl L^{-1})u, u \rangle_{\l2} \leq (1 + \delta) \|u\|_{\l2}^2 
        % \numberthis \label{eq:tildeLinv_L_part1}
        \\
        \stackrel{(1)}{\implies}&\langle u, (L^{-1}\tl) u \rangle_{\l2} \leq (1 + \delta) \|u\|_{\l2}^2\\
        \stackrel{(2)}{\implies}& \langle (L^{-1} \tl -I)u, u\rangle_{\l2} \leq \delta\|u\|_{\l2}^2 
        \numberthis \label{eq:tildeLinv_L_minus_I}
    \end{align*}
    where we use the fact that the operators $\tl$ and $L^{-1}$ are self-adjoint to get (1) 
    and then bring the appropriate terms to the LHS in (2).
\end{proof}

The second property of the 
sequence of functions
is that they converge exponentially fast. Namely, we show: 
\begin{lemma} [Convergence of gradient descent in $L^2$]
    \label{lemma:convergence_proof}
    Let $\uapproxsolution$
    denote the unique solution to the PDE
    $\tl u = \fkeigenapprox$,
    where $\fkeigenapprox \in \Keigenfunctionsapprox$,
    and the operator $\tl$ satisfies the conditions in Lemma~\ref{lemma:bounded_operator_norm}.
    For any $u_0 \in H_0^1(\Omega)$ such that $u_0 \in \Keigenfunctionsapprox$, 
    we define the sequence
    \begin{equation}
        \label{eq:update_equation_convergence_proof}
        u_{t + 1} \leftarrow u_t - \frac{2}{\tlambda_1 + \tlambda_k}(\tl u_t - \fkeigenapprox) \quad (t \in \N)
    \end{equation}
    where for all $t \in \N$, $u_t \in H_0^1(\Omega)$. 
    Then for any $t \in \mathbb{N}$, we have 
    $$\|u_{t} - \uapproxsolution\|_{\l2} 
            \leq \left(\frac{\tlambda_k - \tlambda_1}{\tlambda_k + \tlambda_1}\right)^{t-1} \|u_0 - \uapproxsolution\|_{\l2}$$
    %$\epsilon \geq 0$ we have
    %$\|u_T - \uapproxsolution\|_{\l2} \leq \epsilon$ 
    %after $T$ iterations where, 
    %$$T \geq \frac{\log\left(\frac{\|u_0 - \uapproxsolution\|_{\l2}}{\epsilon}\right)}
    %{\log\left(\frac{\tlambda_k + \tlambda_1}{\tlambda_k - \tlambda_1}\right)}$$
\end{lemma}
The proof is essentially the same as the the analysis of the convergence time of gradient descent for strongly convex losses. Namely, we have:
\begin{proof}
    Given that $u_0 \in \h$ and $u_0 \in \Keigenfunctionsapprox$ 
    the function 
    $\tl u_0 \in \h$ and $\tl u_0 \in \Keigenfunctionsapprox$ as well (from Lemma~\ref{lemma:bounded_operator_norm}).
    
    As $\fkeigenapprox \in \Keigenfunctionsapprox$,
    all the iterates in the sequence will also belong to $\h$ and will lie in the $\Keigenfunctionsapprox$.
    
    Now at a step $t$ the iteration looks like,
    \begin{align*}
        &u_{t + 1} = u_n - \frac{2}{\tlambda_k + \tlambda_1}\left(\tl u_t - \fkeigenapprox\right) \\
        &u_{t + 1} - \uapproxsolution = \left(I - \frac{2}{\tlambda_k + \tlambda_1}\tl\right)(u_t - \uapproxsolution)
    \end{align*}
    Using the result from Lemma~\ref{lemma:bounded_operator_norm}, part \ref{eq:proof_for_upper_bounding_operator_norm_k_eigenfunctions}.
    % (\ref{eq:proof_for_upper_bounding_operator_norm_k_eigenfunctions}) 
    we have,
    \begin{align*}
        & \|u_{t+1} - \uapproxsolution\|_{\l2} 
        \leq  \left(\frac{\tlambda_k - \tlambda_1}{\tlambda_k + \tlambda_1}\right)\|u_t - \uapproxsolution\|_{\l2}
        \\
        &\implies \|u_{t+1} - \uapproxsolution\|_{\l2} 
            \leq \left(\frac{\tlambda_k - \tlambda_1}{\tlambda_k + \tlambda_1}\right)^t \|u_0 - \uapproxsolution\|_{\l2}
    \end{align*} 
    This finishes the proof. 
    \iffalse
    Hence this implies that, $\|u_{T} - \uapproxsolution\|_{\l2} \leq \epsilon$ when
    $$T \geq \frac{\log\left(\frac{\|u_0 - \uapproxsolution\|_{\l2}}{\epsilon}\right)}{\log\left(\frac{\tlambda_k + \tlambda_1}{\tlambda_k - \tlambda_1}\right)}$$
    Using $\kappa:= \frac{\tlambda_k + \tlambda_1}{\tlambda_k - \tlambda_1}$ we can rewrite the above as
    $$T \geq \frac{\log \left(\frac{\|u_0 - \uapproxsolution\|_{\l2}}{\epsilon}\right)}{\log(\kappa)}$$
    \qedhere
    \fi
\end{proof}
% The rate is characterized in the following lemma:

% can be found in Section~\ref{section:appendix_proof_of_lemma3} of the Appendix.

Combining the results 
from 
Lemma~\ref{lemma:error_between_ustar_uapproxsolution}
and Lemma~\ref{lemma:convergence_proof} via triangle inequality, we have:
\begin{align*}
    \|u^\star - u_T\|_{\l2} &\leq \|u^\star - \uapproxsolution\|_{\l2} + \|\uapproxsolution - u_T\|_{\l2}
\end{align*}
and the first term on the RHS subsumes the first three summands of $\tilde{\epsilon}$ defined in Theorem~\ref{thm:main_result}.

\subsection{Approximating iterates by neural networks}
\label{section:network_approximation}

In Lemma \ref{lemma:convergence_proof}, 
we show that there exists a sequence of functions (\ref{eq:update_equation_convergence_proof})
which
converge fast to a function close to $u^\star$. The next step in the proof is to approximate the iterates by neural networks. 

The main idea is as follows. Suppose first the iterates $u_{t+1} = u_t - \eta( \tl u_{t} - \fkeigenapprox)$ are such that $\fkeigenapprox$ is exactly representable as a neural network. Then, the iterate $u_{t+1}$ can be written in terms of three operations performed on $u_t$, $a$ and $f$: taking derivatives, multiplication and addition. Moreover,
if $g$ is representable 
as a neural network with $N$ parameters,
the coordinates of the vector
$\nabla g$ can be represented 
by a neural network with $O(N)$ parameters. 
This is a classic result 
(Lemma \ref{thm:backpropagation}), 
essentially following 
from the backpropagation algorithm. 
Finally,  addition or multiplication of  
two functions representable 
as neural networks with sizes $N_1, N_2$ 
can be represented as neural networks with size $O(N_1 + N_2)$
(see Lemma~\ref{lemma:lemma_operations}).

Using these facts, 
we can write down a recurrence upper bounding the size of 
neural network approximation $u_{t+1}$,
denoted by $\hu_{t+1}$,
in terms of the number of parameters in $\hu_t$ 
(which is the neural network approximation to $u_t$).
Formally, we have: 

\begin{lemma}[Recursion Lemma]
    \label{lemma:recursion_lemma}
    Given the Assumptions \numassumptions,
    consider the update equation 
    \begin{equation}
        \label{eq:update_equation_fnn}
        \hu_{t+1} \leftarrow \hu_t - \frac{2}{\tlambda_1 + \tlambda_k}
        \left(\tl \hu_t - \fnn\right)
    \end{equation}
    If at step $t$, 
    $\hu_t: \R^d \to \R$ 
    is a neural network with $N_t$ parameters,
    then the function $\hu_{t+1}$ is a neural network
    with $O(d^2(N_A + N_t) + N_t + N_{\tilde{f}} + N_c)$ parameters.
\end{lemma}
\begin{proof}
    Expand the update
    $\hu_{t+1} \leftarrow \hu_t - \eta\left(\tl \hu_{t} - \fnn \right) $
    as follows:
    \begin{align*}
        &\hu_{t+1} \leftarrow 
        \hu_t - \eta\left(\sum_{i,j=1}^d \tilde{a}_{ij}\partial_{ij}\hu_t
                  + \sum_{j=1}^d \left(\sum_{i=1}^d \partial_i \tilde{a}_{ij}\right)\partial_j \hu_t + \tilde{c}\hu_t 
                  - \fnn\right).
    \end{align*}
    Using Lemma \ref{thm:backpropagation},
    $\partial_{ij}\hu_t$, $\partial_j \hu_t$ and $\partial_i \tilde{a}_{ij}$
    can be represented by a neural network with $O(N_t)$, $O(N_t)$ and $O(N_A)$ parameters, respectively. 
    Further, $\partial_{i}\tilde{a}_{ij}\partial_j u$ 
    and $\tilde{a}_{ij}\partial_{ij}\hu$ can be represented by a neural network with
    $O(N_A + N_t)$ parameters,
    and $\tilde{c}\hu_t$ can be represented by a network with $O(N_t + N_c)$ parameters,
    from Lemma \ref{lemma:lemma_operations}.
    Hence $\hu_{t+1}$ can be represented in $O(d^2(N_A + N_t) + \fnumparams + N_c + N_t)$ parameters.
    Note that, throughout the entire proofs $O$ hides independent constants.
\end{proof}

Combining the results of Lemma~\ref{lemma:convergence_proof} and Lemma~\ref{lemma:recursion_lemma},
we can get a recurrence for the number of parameters required to represent the
neural network $\hat{u}_t$: 
\begin{align*}
    & N_{t+1} \leq d^2 N_{t} + d^2 N_A + N_{t} + N_{\tilde{f}} + N_c
\end{align*}
Unfolding this recurrence, we get $N_{T} \leq d^{2T}N_0 + \frac{d^2(d^{T} - 1)}{d^2 - 1}N_A + T(\fnumparams) + N_c)$. 

\iffalse
Hence, the total number of parameters required 
for a neural network to approximate a solution to 
a PDE of the form in Definition~\ref{eq:main_pde}
\begin{equation*}
    \begin{split}
        O\Bigg(&d^{2\frac{\log\left(\frac{\|u_0 - \uapproxsolution\|_{\l2}}{\epsilon}\right)}{\log{\kappa}}}\left(N_0 + N_A\right) + 
        \frac{\log\left(\frac{\|u_0 - \uapproxsolution\|_{\l2}}{\epsilon}\right)}{\log{\kappa}}(\fnumparams + N_c)\Bigg)
    \end{split}
\end{equation*}
\fi

The formal lemmas for the different operations on neural networks we can simulate using a new neural network are as follows:

\begin{lemma}[Backpropagation,~\citet{rumelhart1986learning}]
    \label{thm:backpropagation}
    Consider neural network $g:\R^m \rightarrow \R$ with depth $l$, $N$ parameters and differentiable activation functions in the set $\{\sigma_i\}_{i=1}^A$.
    There exists a neural network of size $O(l + N)$ and activation functions in the set $\{\sigma_i, \sigma'_i\}_{i=1}^A$
    that calculates the gradient 
    $\frac{dg}{di}$ for all $i \in [m]$.
\end{lemma}

\begin{lemma}
[Addition and Multiplication]
    \label{lemma:lemma_operations}
    Given neural networks $g:\Omega \rightarrow \R$, $h: \Omega \rightarrow \R$,
    with $N_g$ and $N_h$ parameters respectively, 
    the operations 
    $g(x) + h(x)$ and $g(x) \cdot h(x)$ can
    be represented by neural networks of size $O(N_g + N_h)$,
    and square activation functions.
\end{lemma}
\begin{proof}
    For \emph{Addition}, there exists a network $h$ 
    containing both networks $f$ and $g$ as subnetworks 
    and an extra layer to compute 
    the addition between their outputs. 
    Hence, the total number of parameters 
    in such a network will be $O(N_f + N_g)$.

    For \emph{Multiplication}, consider 
    the operation $f(x) \cdot g(x) = \frac{1}{2}\left(\left(f(x) + g(x)\right)^2 - f(x)^2 - g(x)^2\right)$.
    Then following the same argument as for addition of two networks, 
    we can construct a network $h$ containing
    both networks and square activation function. 
\end{proof}

While the representation result in Lemma \ref{lemma:lemma_operations}
is shown using square activation, 
we refer to \cite{yarotsky2017error} 
for approximation results with ReLU activation. 
The scaling with respect to 
the number of parameters in the network 
remains the same.

\bigskip
Finally, we have to deal with the fact that $\tilde{f}_{\spn}$ is not exactly a neural network, but only approximately so. The error due to this discrepancy can be characterized through the following lemma:

\begin{lemma}[Error using $\fnn$]
    \label{lemma:error_fneural_network}
    Consider the update equation in \eqref{eq:update_equation_fnn},
    where $\fnn$ is a neural network with $N_f$.
    Then the neural network $\hat{u}_t$ approximates the function $u_t$
    such that $\|u_t - \hat{u}_t\|_{\l2} \leq \epsilon^{(t)}_{\nn}$
    where $\epsilon^{(t)}_{\nn}$ is 
    % $$O \left((\max\{1, t^2\eta e C\})^t \left(\epsilon_{\spn} + \epsilon_{\nn} + \frac{2^{3/2} \delta \|f\|_{\l2}}{\gamma - \delta} \right)\right)$$
    % $$O \left((\max\{1, t^2\eta e C\})^t \left(\epsilon_{\spn} + \epsilon_{\nn} + 2\left(1 + \constantdeltagamma\right)\lambda_k^t \|\fkeigen\|_{\l2}\right)\right)$$
    $$O \left((\max\{1, t^2\eta e C\})^t \left(\epsilon_{\spn} + \epsilon_{\nn} 
    + \constantnumber \left(1 + \constantdeltagamma\right)\lambda_k^t \|f\|_{\l2}\right)\right)$$
    where 
    $\delta  = \deltavalue$,
    $\gamma = \gammavalue$,
    and $\alpha$ is a multi-index.
\end{lemma}

The proof for the lemma is deferred to Section~\ref{section:proof_of_lemma_error_fneural_network} of the Appendix.
The main strategy to prove this lemma involves tracking the ``residual'' non-neural-network part of the iterates.
Precisely, for every $t \in \mathbb{N}$, we will write $u_t = \hat{u}_t + r_t$, s.t. $\hat{u}_t$ is a neural network and bound $\|r_t\|_{\l2}$. $\{\hat{u}_t\}_{t=0}^{\infty}$ is defined such that 
    $$
    \begin{cases}
      \hat{u}_0 = u_0, \\
      \hat{u}_{t+1} = \hat{u}_t - \eta\left(\tl\hat{u}_t - \fnn\right)
    \end{cases}
    $$
    Correspondingly, as $r_t = u_t - \hat{u}_t$, we have: 
    $$
    \begin{cases}
      r_0 = 0, \\
      r_{t+1} = (I - \eta \tl) r_t - r
    \end{cases}
    $$

Unfolding the recurrence, we have 
$r_{t} = \sum_{i=0}^{t-1} (I - \eta \tl)^{(i)}r$, which reduces the proof to bounding $\|(I - \eta \tl)^{(i)}\|_{\l2}$. 
\footnote{The reason we require that  
 $\fnn$ is close to $f$ not only in the $L_2$ sense but also in terms of their higher order derivatives is since  $\tl^{(t)} r$
involves $2t$-order derivatives of $r$ to be bounded at each step.}

\section{Applications to Learning Operators}
\label{section:discussion}

A number of recent works attempt 
to simultaneously approximate the solutions 
for an entire family of PDEs 
by learning a \emph{parametric map}
that takes as inputs 
(some representation of)
the coefficients of a PDE 
and returns its solution
\citep{bhattacharya2020model, li2020neural, li2020fourier}. 
For example, given a set of observations 
that $\{a_j, u_j\}_{j=1}^N$, 
where each $a_j$ denotes a coefficient of a PDE 
with corresponding solution $u_j$,
they learn a neural network $G$ 
such that for all $j$, $u_j = G(a_j)$.
Our parametric results provide useful insights 
for why simultaneously solving an entire family
of PDEs with a \emph{single neural network} $G$ 
is possible in the case of linear elliptic PDEs.

Consider the case where the coefficients 
$a_j$ in the family of PDEs are given 
by neural networks with a fixed architecture, 
but where each instance of a PDE is characterized
by a different setting of the weights 
in the models representing the coefficients.
Lemma~\ref{lemma:recursion_lemma}
shows that each iteration of our sequence 
(\ref{eq:update_equation_convergence_proof})
constructs a new network 
containing both the current solution 
and the coefficient networks as subnetworks. 
%In the resulting construction,
%each coefficient network appears many times
%as a repeated motif. 
We can view our approximation as not merely
approximating the solution to a single PDE 
but to every PDE in the family,
%where the coefficients of the subnetworks 
%are specified as inputs
%and 
%to the final network.
%Hence, treating the coefficient networks
%as placeholders to be specified as inputs,
%our construction can be interpreted 
%as providing a parametric map
%between the coefficients of an elliptic PDE and its solution.
by treating the coefficient networks 
as placeholder architectures 
whose weights are provided as inputs.
Thus, our construction 
provides a parametric map 
between the coefficients 
of an elliptic PDE in this family and its solution.

% Recall that in Lemma~\ref{lemma:recursion_lemma}
% we show that each point of update equation (\ref{eq:update_equation_convergence_proof})
% includes the repeated operation of the neural networks that approximate the coefficients of the PDE.
% For our construction we consider these weights to be fixed throughout, 
% and can be treated as weight tied blocks in the final network.
% However,
% we can very well consider that the weights for a coefficient to be a variable that takes its value as an input to the final network.
% Guided by Lemma~\ref{lemma:recursion_lemma},
% weights for different instantiations of a coefficient can be plugged  
% in final neural network appropriately.
% Hence, our construction remains the same 
% and our result can be generalized to networks that learn a parametric map between coefficients of an elliptic PDE to its solutions.

% \paragraph{Conclusion and Future work}
\section{Conclusion and Future Work}
We derive parametric complexity bounds 
for neural network approximations 
for solving linear elliptic PDEs 
with Dirichlet boundary conditions,
whenever the coefficients 
can be approximated by
are neural networks 
with finite parameter counts. %number of parameters.
By simulating gradient descent in function spaces
using neural networks,
we construct a neural network
that approximates the solution of a PDE.
We show that the number of parameters
in the neural network 
depends on the parameters required
to represent the coeffcients
and has a $\mbox{poly}(d)$ dependence 
on the dimension of the input space, 
therefore avoiding the curse of dimensionality.

% \blue{An immediate avenue for future work for our paper is related to the 
% need for the neural network to be a function that mostly lies within the span of the top k eigenfunctions of the operator $L$.
% As seen in Lemma~\ref{lemma:error_fneural_network},
% the error 
% can potentially have
% an exponential dependence on the number of iterates in the sequence.
% We feel that this condition can be relaxed with some kind of ``regularity'' assumptions, 
% but we might imagine that weaker conditions could suffice.
% }
% The most immediate open questions 
% are related to tightening our results: 
% can the assumption that $\tilde{f}$ lies 
% % within the span of the first $k$ eigenfunctions be relaxed? 
% It seems like some kind of ``regularity'' assumption is needed, 
% but we might imagine that weaker conditions could suffice.
An immediate open question is related 
to the tightening our results:
our current error bound
is sensitive to the neural network approximation 
lying close to $\Keigenfunctions$ which could be alleviated by
relaxing 
(by adding some kind of ``regularity'' assumptions)
the dependence of our analysis on the first $k$ eigenfunctions.
Further, the dependencies in the exponent 
of $d$ on $R$ and $\kappa$ 
in parametric bound may also be improvable. 
Finally, the idea of simulating an iterative algorithm 
by a neural network to derive a representation-theoretic result 
is broadly
applicable,
and may be a fertile ground for further work, 
both theoretically and empirically,
as it suggest a particular kind of weight tying.

\newpage
\bibliographystyle{plainnat}
%\bibliography{bibliography.bib}

% \clearpage
\appendix
\section{Brief Overview of Partial Differential Equations}
In this section, we introduce few key definitions and results 
from PDE literature.
We note that the results in this section are standard and 
have been included in the Appendix for completeness. 
We refer the reader to classical texts on PDEs \citep{evans1998partial,gilbarg2001elliptic} for more details.

We will use the following Poincar\'e inequality throughout our proofs.
\begin{theorem}[Poincar\'e inequality]
    \label{thm:poincare_inequality}
    Given $\Omega \subset \R^d$, a bounded open subset, 
    there exists a constant $\pc > 0$ such that for all $u \in H_0^1(\Omega)$
    $$\|u\|_{L^2(\Omega)} \leq \pc \|\nabla u\|_{L^2(\Omega)}.$$
\end{theorem}

\begin{corollary} 
    \label{corollary:equivalence}
    For the bounded open subset $\Omega \subset \R^d$, 
    for all $u \in \h$, 
    we define the norm in the Hilbert space $\h$ as
\begin{equation}
    \label{eq:H_0_1_norm}
    \|u\|_{\h} = \|\nabla u\|_{\l2}.
\end{equation}
Further, the norm in $\h$ is equivalent to the norm $\hh$.
\end{corollary}
\begin{proof}
    Note that for $u \in \h$ we have, 
    \begin{align*}
        \|u\|_{\hh} &= \|\nabla u\|_{\l2} + \|u\|_{\l2} \\
        &\geq \|\nabla u\|_{\l2} \\
        \implies \|u\|_{\hh} &\geq \|u\|_{\h}.  
        % \numberthis \label{eq:H_01_lbd_H_1}
    \end{align*}
    Where we have used the definition of the norm in $\h$ space. 
    
    Further, using the result in Theorem \ref{thm:poincare_inequality} 
    we have
    \begin{align*}
        \|u\|^2_{\hh} = 
            \left(\|u\|^2_{\l2} + \|\nabla u\|^2_{L^2(\Omega)}\right) 
            \leq \left(C_p^2 + 1\right)\|\nabla u\|^2_{\hh}
        \numberthis \label{eq:upd_H1_H_0_1}
    \end{align*}
    Therefore, 
    % combining the results from (\ref{eq:upd_H1_H_0_1}) and (\ref{eq:H_01_lbd_H_1}) we have
    combining the two inequalities we have
    \begin{equation}
        \label{eq:relation_H_01_H_1}
        \|u\|_{\h} \leq \|u\|_{\hh} \leq C_h\|u\|_{\h}
    \end{equation}
    where $C_h = (C_p^2  + 1)$.
    Hence we have that the norm in $\h$ and $\hh$ spaces are equivalent.
\end{proof}

\begin{proposition}[Equivalence between $\l2$ and $\h$ norms]
    \label{proposition:equivalence_L2_H01}
    If $v \in \Keigenfunctions$ then we have that $\|v\|_{\l2}$ is equivalent to $\|v\|_{\h}$.
\end{proposition}
\begin{proof}
    We have from the Poincare inequality in Theorem~\ref{thm:poincare_inequality} 
    that for all $v \in \h$, the norm in $\l2$ is upper bounded by the norm in $\h$, i.e.,
    \begin{align*}
        \|v\|^2_{\l2} \leq \|v\|^2_{\h}
    \end{align*}
    Further, using results from \eqref{eq:bilinear_form_lowerbound} and \eqref{eq:upperbound_bilinear_form_lax_milgram} 
    (where $b(u, v):= \langle Lu, v\rangle_{\l2}$), 
    we know that for all $v \in \h$ we have
    \begin{align*}
        m\|v\|_{\h}^2 \leq \langle Lv, v\rangle_{\l2} \leq \max\{M, C_p\|c\|_{\linf}\}\|v\|_{\h}^2
    \end{align*}
    This implies that $\langle Lu, v\rangle_{\l2}$ is equivalent to 
    the inner product $\langle u, v\rangle_{\h}$, i.e., 
    % there exist constants $c_1$ and $c_2$ such that 
    for all $u, v \in \h$,
    % $$ c_1 \langle u, v \rangle_{\h} \leq \langle Lu, v \rangle_{\l2} \leq c_2\langle u, v\rangle_{\h}$$
    $$ m \langle u, v \rangle_{\h} \leq \langle Lu, v \rangle_{\l2} 
        \leq \max\left\{M, C_p\|c\|_{\linf}\right\}\langle u, v\rangle_{\h}$$
    Further, since $v \in \Keigenfunctions$, we have from Lemma~\ref{lemma:bounded_operator_norm} that
    \begin{align*}
     &\langle Lv, v\rangle_{\l2} \leq \lambda_k \|v\|_{\l2}^2\\
     \implies& \|v\|_{\h} \leq \frac{\lambda_k}{c_1}\|v\|^2_{\l2}
    \end{align*}
    Hence we have that for all $v \in \Keigenfunctions$ $\|v\|_{\l2}$ is equivalent to $\|v\|_{\h}$ 
    and by Corollary~\ref{corollary:equivalence} 
    is also equivalent to
    $\|v\|_{\hh}$.
\end{proof}

Now introduce a form for $\langle Lu, v\rangle_{\l2}$ that is more amenable
for the existence and uniqueness results.
\begin{lemma}
    \label{proposition:weak_solution}
    For all $u, v \in \h$, we have the following,
    \begin{enumerate}
        \item The inner product $\langle Lu, v \rangle_{\l2}$ equals,
            $$\langle Lu, v \rangle_{\l2} = \int_\Omega \left(A \nabla u \cdot \nabla v + cuv \right)\; dx $$
        \item The operator $L$ is self-adjoint.
    \end{enumerate}
\end{lemma}
\begin{proof}
    \begin{enumerate}
        \item 
        We will be using the following integration by parts formula,
        $$ 
            \int_\Omega \frac{\partial u}{\partial x_i} dx = 
            -\int_\Omega u \frac{\partial v}{\partial x_i} dx + \int_{\partial \Omega} uv n_i \partial \Gamma 
        $$
        Where $n_i$ is a normal at the boundary and $\partial \Gamma$ is an infinitesimal element of the boundary.
        
        Hence we have for all $u, v \in \h$, 
        \begin{align*}
            \langle Lu, v \rangle_{\l2} &= \int_\Omega -\left(\sum_{i=1}^d \left(\partial_i\left(A\nabla u\right)_i\right)\right) v + cuv \; dx \\
            &= \int_\Omega A \nabla u \cdot \nabla v dx - \int_{\partial \Omega}\left(\sum_{i=1}^{d}(A \nabla u)_i n_i\right) v d\Gamma + \int_\Omega cuv dx \\
            &= \int_\Omega A \nabla u \cdot \nabla v dx + \int_\Omega cuv dx \qquad (\because v_{|\partial \Omega} = 0)
        \end{align*}
    \item 
    To show that the operator $L: \h \to \h$ is self-adjoint,
    we show that for all $u, v \in \h$ we have $\langle Lu, v \rangle = \langle u, Lv\rangle$.
    
    From Proposition~\ref{proposition:weak_solution}, for functions $u, v \in \h$ we have
    \begin{align*}
        \langle Lu, v \rangle_{\l2} &= \int_\Omega A \nabla u \cdot \nabla v dx + \int_{\Omega} cuv dx\\
         &= \int_\Omega A \nabla v \cdot \nabla u dx + \int_{\Omega} cvu dx \\
        %  &= \langle Lv, u \rangle
         &= \langle u, Lv \rangle
    \end{align*}
    \end{enumerate}
\end{proof}

\subsection{Proof of Proposition~\ref{p:maingd}}
\label{section:Proof_of_proposition_1}
We first show that if $u$ is the unique solution then it minimizes the variational norm.
    
    Let $u$ denote the weak solution, further for all $w \in \h$ let $v = u + w$. 
    Using the fact that $L$ is self-adjoint (as shown in Lemma~\ref{proposition:weak_solution})
    we have
    \begin{align*}
        J(v) = J(u + w) 
        &= \frac{1}{2} \langle L(u + w), (u + w) \rangle_{\l2} - \langle f, u + w\rangle_{\l2}\\
        &= \frac{1}{2}\langle Lu, u\rangle_{\l2} 
            + \frac{1}{2}\langle Lw, w \rangle_{\l2}
            + \langle Lu, w\rangle_{\l2} - \langle f, u\rangle_{\l2} 
            -  \langle f, w\rangle_{\l2}\\
        &= J(u) + \frac{1}{2}\langle Lw, w\rangle_{\l2} 
            + \langle Lu, w\rangle_{\l2} - \langle f, w\rangle_{\l2}\\
        &\geq J(u)
    \end{align*}
    where we use the fact that $\langle Lu, u\rangle_{\l2} > 0$
    and that $u$ is a weak solution hence
    \eqref{eq:variational_formulation} holds for all $w \in \h$.
    
    To show the other side, assume that $u$ minimizes $J$, i.e., 
    for all $\lambda > 0$ and $v \in \h$ 
    we have,
    $J(u + \lambda v) \geq J(u)$,
    \begin{align*}
        &J(u + \lambda v) \geq J(u)\\
        &\frac{1}{2}\langle L(u + \lambda v), (u + \lambda v)\rangle_{\l2}
           - \langle f, (u + \lambda v) \rangle_{\l2} 
           \geq \frac{1}{2}\langle Lu, u\rangle_{\l2} - \langle f, u\rangle_{\l2}\\
        \implies&
             \frac{\lambda}{2}\langle Lv, v \rangle_{\l2}
            + \langle Lu, v\rangle_{\l2}
            - \langle f, v\rangle_{\l2} \geq 0
    \end{align*}
    Taking $\lambda \to 0$, we get
    $$ \langle Lu, v\rangle_{\l2} - \langle f, v\rangle_{\l2} \geq 0$$
    and also taking $v$ as $-v$, we have
    $$ \langle Lu, v\rangle_{\l2} - \langle f, v\rangle_{\l2} \leq 0$$
    Together, this implies that 
    if $u$ is the solution to \eqref{eq:minimization_problem}, then 
    $u$ is also the weak solution, i.e, for all $v \in \h$ we
    have
    $$ \langle Lu, v \rangle_{\l2} = \langle f, v\rangle_{\l2}$$

\subsection*{Proof for Existence and Uniqueness of the Solution}
In order to prove for the uniqueness of the solution, 
we first state the Lax-Milgram theorem.
\begin{theorem}[Lax-Milgram,~\citet{lax1954parabolic}]
    \label{thm:lax_milgram}
    Let $\mathcal{H}$ be a Hilbert space with inner-product 
    $(\cdot, \cdot): \mathcal{H}\times \mathcal{H} \rightarrow \R$, 
    and let
    $b: \mathcal{H} \times \mathcal{H}\rightarrow \R$ and 
    $l: \mathcal{H} \rightarrow \R$ be 
    the bilinear form and linear form, respectively.
    Assume that there exists constants $C_1, C_2, C_3 > 0$ such that for all $u,v \in \mathcal{H}$ 
    we have,
    $$C_1 \|u\|^2_{\ch}\leq b(u, u), \quad |b(u, v)| \leq C_2\|u\|_{\ch}\|v\|_{\ch}, \quad \text{and\;}  |l(u)|\leq C_3\|u\|_{\ch}.$$
    Then there exists a unique $u\in \mathcal{H}$ such that, 
    $$b(u,v) = l(v) \quad \text{for all}\; v \in \mathcal{H}.$$
\end{theorem}

Having stated the Lax-Milgram Theorem, we make the following proposition,
\begin{proposition}
    \label{proposition:unique_solution_proof}
    Given the assumptions \numassumptions, 
    solution to the variational formulation in Equation~\ref{eq:variational_formulation}
    exists and is unique.
\end{proposition}
\begin{proof}
    Using the \emph{variational formulation} defined in (\ref{eq:variational_formulation}), we introduce the 
    bilinear form $b(\cdot, \cdot): \h \times \h \to \R$ where $b(u, v) := \langle Lu, v\rangle$.
    Hence, we prove the theorem by showing that the bilinear form $b(u, v)$
    satisfies the conditions in Theorem~\ref{thm:lax_milgram}.
    
    We first show that for all $u, v \in \h$ the following holds,
   \begin{align*}
       |b(u,v)| &= \left| \int_{\Omega} \left(A \nabla u \cdot \nabla v + cuv\right) dx\right| \\
                &\leq \int_{\Omega} \left| \left(A\nabla u \cdot \nabla v + cuv\right)\right| dx \\ 
                &\leq \int_{\Omega} \left| A\nabla u \cdot \nabla v\right| dx + \int_\Omega \left| cuv\right| dx \\ 
                &\leq \|A\|_{L^\infty(\Omega)}\|\nabla u\|_{L^2(\Omega)}\|\nabla v\|_{L^2(\Omega)} + \|c\|_{L^\infty(\Omega)}\|u\|_{L^2(\Omega)}\|v\|_{L^2(\Omega} \\
                &\leq M\|\nabla u\|_{L^2(\Omega)}\|\nabla v\|_{L^2(\Omega)} 
                    + \|c\|_{L^\infty(\Omega)}\|u\|_{L^2(\Omega)}\|v\|_{L^2(\Omega)} \\
                &\leq \max\left\{M, C_p\|c\|_{L^\infty(\Omega)}\right\}\|u\|_{\h}\|v\|_{\h}
                \numberthis \label{eq:upperbound_bilinear_form_lax_milgram}
   \end{align*}
   
   \noindent 
   Now we show that the bilinear form  $a(u, u)$ is lower bounded.
   \begin{align*}
       b(v,v) &= \int_{\Omega} \left(A \nabla v \cdot \nabla v + cv^2\right) dx \\
              &\geq m \int_\Omega \|\nabla v\|^2 dx  = m\|v\|_{\h}
              \numberthis \label{eq:bilinear_form_lowerbound}
            %   &\geq m \int_\Omega \|\nabla v\|^2 dx   = \lambda \|\nabla v\|_{\h}
            %   &\geq \min(\lambda, \eta)\left(\|\nabla v\|^2_{L^2(\Omega)} + \|v\|^2_{L^2(\Omega)}\right) = \min(\lambda, \eta)\|v\|^2_{H^1(\Omega)}
   \end{align*}
   Finally, for $v \in \h$
   \begin{align*}
       |(f, v)| = \left|\int_\Omega fv dx \right| \leq \|f\|_{L^2(\Omega)}\|v\|_{\l2} \leq C_p\|f\|_{\l2} \|v\|_{\h}
   \end{align*}
   
   \noindent 
   Hence, we satisfy the assumptions in required in Theorem \ref{thm:lax_milgram} and therefore
   the variational problem defined in (\ref{eq:variational_formulation}) has a unique solution.
\end{proof}

% ==== Start Uncommenting ====
% \section{Missing Proofs for Section~\ref{section:main_result}}
% \label{section:appendix_main_result}
% ==== End Uncommenting ====

% \section{Missing Proofs for Section~\ref{section:proof_of_main_result}}
% \label{section:appendix_proof_of_main_result}
% \subsection{Proof for Lemma~\ref{lemma:bounded_operator_norm}}
% \label{section:proof_for_lemma1}
% % \input{sections/Proofs/proof_of_lemma1}

% % ==== Start Uncommenting ====
% \subsection{Proof of Lemma \ref{lemma:convergence_proof}}
% \label{section:appendix_proof_of_lemma3}
% \input{sections/Proofs/proof_of_lemma3}
% % ==== End Uncommenting ====

% \subsection{Important Lemmas for Section~\ref{section:network_approximation}}

\section{Perturbation Analysis}
\label{section:error_analysis}

\subsection{Proof of Lemma~\ref{lemma:error_between_ustar_uapproxsolution}}
\label{section:lemma_error_between_ustar_uapproxsolution}
\begin{proof}
    Using the triangle inequality the error between $u^\star$ and $\uapproxsolution$, we have,
    \begin{equation}
        \label{eq:error_analysis_ustar_uapproxsolution_main}
        \|\ustar - \uapproxsolution\|_{\l2} 
        \leq \underbrace{\|u^\star - \uspansolution\|_{\l2}}_{(I)} 
        + \underbrace{\|\uspansolution - \uapproxsolution\|_{\l2}}_{(II)}
    \end{equation}
    where $\uspansolution$ is the solution to the PDE $Lu = \fkeigen$.
    
    In order to bound Term (I), we use the inequality in \eqref{eq:operator_inequality} to get,
    \begin{align*}
        \|\ustar - \uspansolution\|_{\l2}^2 
        &\leq \frac{1}{\lambda_1} \langle L(\ustar - \uspansolution), \ustar - \uspansolution\rangle_{\l2} \\
        &= \frac{1}{\lambda_1} \langle f - \fkeigen, \ustar - \uspansolution\rangle_{\l2} \\
        &\leq \frac{1}{\lambda_1} \|f - \fkeigen\|_{\l2}\| \ustar - \uspansolution\|_{\l2} \\
        \implies \|\ustar - \uspansolution\|_{\l2} 
        &\leq \frac{1}{\lambda_1} \|f - \fkeigen\|_{\l2} \leq \frac{\epsilon_{\spn}}{\lambda_1} 
        \numberthis \label{eq:Term_I_error}
    \end{align*}
    We now bound Term (II).
    
    First we introduce an intermediate PDE $Lu = \fkeigenapprox$, and denote the solution $\usemieigen$.
    Therefore, by utilizing triangle inequality again Term (II) can be expanded as the following,
    \begin{equation}
        \label{eq:uspansolution_uapproxsolution_error_triangle_inequality}
        \|\uspansolution - \uapproxsolution\|_{\l2} 
            \leq \|\uspansolution - \usemieigen\|_{\l2}
            + \|\usemieigen - \uapproxsolution\|_{\l2}
    \end{equation}
    We will tackle the second term in \eqref{eq:uspansolution_uapproxsolution_error_triangle_inequality} first.
    
    Using $\usemieigen = L^{-1}\fkeigenapprox$ and $\uapproxsolution = \tl^{-1}\fkeigenapprox$,
    \begin{align*}
       \|\usemieigen - \uapproxsolution\|_{\l2} 
       &= \|(L^{-1} - \tl^{-1})\fkeigenapprox\|_{\l2} \\
       &= \|(L^{-1}\tl - I)\tl^{-1}\fkeigenapprox\|_{\l2} \\
       \implies 
       \|\usemieigen - \uapproxsolution\|_{\l2} 
       &= \|(L^{-1}\tl - I)\uapproxsolution\|_{\l2}  
       \numberthis \label{eq:usemieigen_uapproxsolution_l2_step1}
    \end{align*}
    % Therefore using the result from Lemma~\ref{lemma:L_tl_inner_product_close} 
    % the following holds for all $u \in \h$:
    % \begin{align*}
    %     &\langle (\tl L^{-1})u, u \rangle_{\l2} \leq (1 + \delta) \|u\|_{\l2}^2 
    %     % \numberthis \label{eq:tildeLinv_L_part1}
    %     \\
    %     \stackrel{(1)}{\implies}&\langle u, (L^{-1}\tl) u \rangle_{\l2} \leq (1 + \delta) \|u\|_{\l2}^2\\
    %     \stackrel{(2)}{\implies}& \langle (L^{-1} \tl -I)u, u\rangle_{\l2} \leq \delta\|u\|_{\l2}^2 
    %     \numberthis \label{eq:tildeLinv_L_minus_I}
    % \end{align*}
    % where we use the fact that the operators $\tl$ and $L^{-1}$ are self-adjoint to get (1) 
    % and then bring the appropriate terms to the LHS in (2).
    Therefore, 
    using the inequality 
    in Lemma~\ref{lemma:L_tl_inner_product_close} part (2.)
    % and inequality in 
    % Lemma~\ref{lemma:bounded_operator_norm}
    % (with $L^{-1}L$ as the operator),
    we can upper bounded 
    \eqref{eq:usemieigen_uapproxsolution_l2_step1}
    to get,
    \begin{equation}
        \label{eq:usemieigen_uapproxsolution_l2_step2}
        \|\usemieigen - \uapproxsolution\|_{\l2} \leq \delta \|\uapproxsolution\|_{\l2}
    \end{equation}
    where 
    $\delta =\deltavalue$.
    
    Proceeding to the first term in \eqref{eq:uspansolution_uapproxsolution_error_triangle_inequality}, 
    using Lemma~\ref{lemma:fkeigen_fkeigenapprox_davis_kahan},
    and the inequality in \eqref{eq:operator_inequality},
    the term $\|\uspansolution - \usemieigen\|_{\l2}$ can be upper bounded by,
    \begin{align*}
        \|\uspansolution - \usemieigen\|^2_{\l2} 
        &\leq \frac{1}{\lambda_1}\langle L(\uspansolution - \usemieigen), \uspansolution - \usemieigen \rangle_{\l2}\\
        &\leq \frac{1}{\lambda_1}\langle \fkeigen - \fkeigenapprox, \uspansolution - \usemieigen \rangle_{\l2}\\
        &\leq \frac{1}{\lambda_1}\| \fkeigen - \fkeigenapprox\|_{\l2} \|\uspansolution - \usemieigen \|_{\l2}\\
        \implies \|\uspansolution - \usemieigen\|_{\l2} 
            &\leq 
            \frac{1}{\lambda_1}\| \fkeigen - \fkeigenapprox\|_{\l2}
            \leq \frac{\delta}{\lambda_1} \cdot \errordaviskahan
        \numberthis \label{eq:uspansolution_usemigeigen_l2}
    \end{align*}
    
    Therefore Term (II), i.e., $\|\uspansolution - \uapproxsolution\|_{\l2}$
    can be upper bounded by 
    \begin{equation}
        \label{eq:Term_II_error}
        \|\uspansolution - \uapproxsolution\|_{\l2} 
        \leq \|\uspansolution - \usemieigen\|_{\l2} + \|\usemieigen - \uapproxsolution\|_{\l2}
        \leq \frac{\hat{\epsilon}_f}{\lambda_1} + \delta \|\uapproxsolution\|_{\l2}
    \end{equation}
    
    Putting everything together, we can upper bound \eqref{eq:error_analysis_ustar_uapproxsolution_main} as
    \begin{align*}
        \|\ustar - \uapproxsolution\|_{\l2} 
        &\leq \|u^\star - \uspansolution\|_{\l2}
        + \|\uspansolution - \uapproxsolution\|_{\l2} \\
        %\implies
        %\|\ustar - \uapproxsolution\|_{\l2} 
        &\leq \frac{\epsilon_{\spn}}{\lambda_1}
        %   + \frac{\hat{\epsilon}_f}{\lambda_1} 
          + \frac{\delta}{\lambda_1} \errordaviskahan
          + \delta \|\uapproxsolution\|_{\l2}
    \end{align*}
    where 
    $\gamma = \gammavalue$ 
    and 
    $\delta =\deltavalue$.
\end{proof}

\subsection{Proof of Lemma~\ref{lemma:error_fneural_network}}
\label{section:proof_of_lemma_error_fneural_network}
\begin{proof}
    We define $r = \fkeigenapprox - \fnn$, therefore from Lemma~\ref{lamma:fnn_feigenapprox_close}
    we have that for any multi-index $\alpha$, 
    $$\|\tl^{t}r\|_{\l2} \leq (t!)^2 \cdot C^t \left(\epsilon_{\nn} + \epsilon_{\spn}\right) 
    + \constantnumber
    \left(1 + \constantdeltagamma\right) \lambda_k^t \|\fkeigen\|_{\l2}.$$
    
    For every $t \in \mathbb{N}$, we will write $u_t = \hat{u}_t + r_t$, s.t. $\hat{u}_t$ is a neural network and we (iteratively) bound $\|r_t\|_{\l2}$. 
    Precisely, we define a sequence of neural networks $\{\hat{u}_t\}_{t=0}^{\infty}$, s.t. 
    $$
    \begin{cases}
      \hat{u}_0 = u_0, \\
      \hat{u}_{t+1} = \hat{u}_t - \eta\left(\tl\hat{u}_t - \fnn\right)
    \end{cases}
    $$
    Since $r_t = u_t - \hat{u}_t$, we can define a corresponding recurrence for $r_t$: 
    $$
    \begin{cases}
      r_0 = 0, \\
      r_{t+1} = (I - \eta \tl) r_t - r
    \end{cases}
    $$
    
    Unfolding the recurrence, we get
    \begin{align*}
        r_{t+1} &= \sum_{i=0}^{t} (I - \eta \tl)^{i}r 
        \numberthis \label{eq:u_tp1_final}
    \end{align*}
    Using the binomial expansion we can write:
    \begin{align*}
        (I - \eta \tl)^{t}r &= \sum_{i=0}^{t} \binom{t}{i}(-1)^i(\eta \tl)^{i}r \\
        \implies \|(I - \eta \tl)^{(t)}r\|_{\l2} 
        &= \left\|\sum_{i=0}^t \binom{t}{i} (-1)^i (\eta \tl)^{i}r\right\|_{\l2} \\
        &\leq \sum_{i=0}^t \binom{t}{i} \eta^i\|\tl^{i}r\|_{\l2}\\
        % &\stackrel{(1)}{\leq} 
        &\leq 
        \sum_{i=0}^t \left(\frac{t e}{i}\right)^i \eta^i\|\tl^{i}r\|_{\l2}
        \qquad \because \text{$\binom{t}{i} \leq \left(\frac{te}{i}\right)^i$}
        \\
        &\stackrel{(1)}{\leq} 
        % &\leq
        \sum_{i=0}^t \left(\frac{te}{i} \eta\right)^i
            \left( (i!)^2 C^i \left(\epsilon_{\nn} + \epsilon_{\spn}\right) + 
            \constantnumber\left(1 + \constantdeltagamma\right)\lambda_k^i \|\fkeigen\|_{\l2}
            \right)
            \\
            % \quad \text{from Lemma~\ref{lamma:fnn_feigenapprox_close}}\\
        &\leq \sum_{i=0}^t \left(\frac{te}{i} \eta\right)^i
            (i!)^2 C^i
            \left( \left(\epsilon_{\nn} + \epsilon_{\spn}\right) 
            % + \frac{\lambda_k^i}{(i! C^i)} \errorspanstar\right)\\
            + \constantnumber\left(1 + \constantdeltagamma\right)\frac{\lambda_k^i}{(i!)^2 C^i} \|\fkeigen\|_{\l2}\right)\\
        &\stackrel{(2)}{\leq} 
        % &\leq 
        \sum_{i=0}^t \left(\frac{te}{i} \eta i^2 C\right)^i  
            \left( \left(\epsilon_{\nn} + \epsilon_{\spn}\right) 
            + \constantnumber\left(1 + \constantdeltagamma\right)\frac{\lambda_k^i}{(i!)^2 C^i} \|\fkeigen\|_{\l2}\right)
            % \quad \because i! \leq i^i\\
            \\
        &\stackrel{(3)}{\leq} 
        % &\leq 
        \sum_{i=0}^t \left(\frac{te}{i} \eta i^2 C\right)^i  
            \left( \left(\epsilon_{\nn} + \epsilon_{\spn}\right) 
            + \constantnumber\left(1 + \constantdeltagamma\right)\lambda_k^i \|\fkeigen\|_{\l2}\right)
            % \quad \because \frac{1}{(i!)^2C^i} \leq 1\\
            \\
        &\leq \sum_{i=0}^t (ti e\eta C)^i  
            \left( \left(\epsilon_{\nn} + \epsilon_{\spn}\right) 
            + \constantnumber\left(1 + \constantdeltagamma\right)\lambda_k^i \|\fkeigen\|_{\l2}\right)
            \\
        &\leq t\max\{1,(t^2e\eta C)^t\}
            \left( \left(\epsilon_{\nn} + \epsilon_{\spn}\right) 
            + \constantnumber\left(1 + \constantdeltagamma\right)\lambda_k^t \|\fkeigen\|_{\l2}\right)
    \end{align*}
    Here the inequality $(1)$ follows by using the bound derived in Lemma~\ref{lamma:fnn_feigenapprox_close}.
    Further, we use that all $i \in \N$ we have $i! \leq i^i$ in $(2)$ and the inequality $(3)$ 
    follows from the fact that $\frac{1}{(i!)^2C^i} \leq 1$. 
    
    Hence we have the the final upper bound:
    \begin{align*}
        \|r_{t}\|_{\l2} 
        \leq t^2 \max\{1, (t^2e\eta C)^t\} \left(\epsilon_{\nn} + \epsilon_{\spn} 
        + \constantnumber \left(1 + \constantdeltagamma\right)\lambda_k^t \|\fkeigen\|_{\l2}\right)
    \end{align*}
\end{proof}

\section{Technical Lemmas: Perturbation Bounds}
In this section we introduce some useful lemmas about perturbation bounds used in the preceding parts of the appendix. 

First we show a lemma that's ostensibly an 
application of Davis-Kahan to the (bounded) operators $L^{-1}$ and $\tl^{-1}$. 
% For a function $g \in \h$ we define an operator $P_k$ (and $\tilde{P}_k$)
% as the projection onto $\Keigenfunctions$ (and $\Keigenfunctionsapprox$) as the following:
% \begin{equation}
%     P_kg = \sum_{i=1}^k \langle g, \varphi_i\rangle_{\l2} \varphi_i,
%     \;\;\text{\&}
%     \; \; \; 
%     \tilde{P}_kg = \sum_{i=1}^k \langle g, \tvarphi_i\rangle_{\l2} \tvarphi_i 
% \end{equation}

% \begin{lemma}[Subspace alignment]
%     \label{lemma:perturbation_bound}
%     Consider linear elliptic operators $L$ and $\tl$ 
%     with eigenvalues 
%     $\lambda_1 \leq \lambda_2 \leq \cdots$ and 
%     $\lambda_1 \leq \lambda_2 \leq \cdots$ respectively. 
%     Assume that $\gamma:= \frac{1}{\lambda_k} - \frac{1}{\lambda_{k+1}} > 0$. 
%     Then, there exists an orthogonal transformation $O: \h \to \h$ such that
%     the first $k$ eigenfunctions of $L$ and $\tl$ satisfy,
%     \begin{equation}
%         \label{eq:davis_kahan_form_2}
%         \sup_{\substack{a \in \R^k \\ \sum_{i=1}^k a^2_i = 1}}
%         \left\| \sum_{i=1}^k (O\tvarphi_i - \varphi_i)a_i\right\|_{\l2}
%         \leq \frac{2^{3/2} \delta}{\gamma - \delta}
%     \end{equation}
%     where $\delta = \deltavalue$.
% \end{lemma}
\begin{lemma}[Subspace alignment]
    \label{lemma:perturbation_bound}
    Consider linear elliptic operators $L$ and $\tl$ 
    with eigenvalues 
    $\lambda_1 \leq \lambda_2 \leq \cdots$ and 
    $\lambda_1 \leq \lambda_2 \leq \cdots$ respectively. 
    Assume that $\gamma:= \frac{1}{\lambda_k} - \frac{1}{\lambda_{k+1}} > 0$. 
    For any function $g \in \h$, 
    we define
    $P_k g := \sum_{i=1}^k \langle g, \varphi_i\rangle_{\l2} \varphi_i$
    and 
    $\tilde{P}_k g := \sum_{i=1}^k \langle g, \tvarphi_i\rangle_{\l2} \tvarphi_i $
    as the projection of 
    $g$ onto $\Keigenfunctions$ and $\Keigenfunctionsapprox$,
    respectively. Then we have:
    \begin{equation}
        \|P_k g - \tilde{P}_k g \|_{\l2} \leq \frac{\delta}{\gamma - \delta}\|g\|_{\l2}
    \end{equation}
    where $\delta = \deltavalue$.
\end{lemma}
\begin{proof}
    We begin the proof by first showing that the inverse of the operators
    $L$ and $\tl$ are close.
    Using the result from Lemma~\ref{lemma:L_tl_inner_product_close}
    with 
    $\delta = \deltavalue$, 
    we have:
    \begin{align*}
        \langle (L^{-1}\tl - I)u, u\rangle_{\l2} \leq \delta\|u\|_{\l2}^2 \\
        \implies \langle (L^{-1} - \tl^{-1})\tl u, u\rangle_{\l2} \leq \delta\|u\|_{\l2}^2 \\
        \implies \langle (L^{-1} - \tl^{-1})v, u\rangle_{\l2} \leq \delta\|u\|_{\l2}^2 
    \end{align*}
    Now, the operator norm $\|L^{-1} - \tl^{-1}\|$ can be written as,
    \begin{equation}
        \label{eq:operator_norm_L_tilde_L_inverse}
        \|L^{-1} - \tl^{-1}\| 
        = \sup_{v \in \h} \frac{\langle (L^{-1} - \tl^{-1})v, v\rangle_{\l2}}{\|v\|^2_{\l2}}
        \leq \delta
    \end{equation}
    Further note that, $\{\frac{1}{\lambda_i}\}_{i=1}^\infty$ and 
    $\{\frac{1}{\tlambda_i}\}_{i=1}^\infty$ are the eigenvalues of the operators 
    $L^{-1}$ and $\tl^{-1}$, respectively. 
    Therefore from \textit{Weyl's Inequality} 
    and \eqref{eq:operator_norm_L_tilde_L_inverse}
    we have:
    \begin{equation}
        \label{eq:weyls_inequality}
        \sup_{i} \left|\frac{1}{\lambda_i} - \frac{1}{\tlambda_i}\right| \leq \|L^{-1} - \tl^{-1}\| \leq \delta
    \end{equation}
    Therefore, for all $i \in \N$, 
    we have that 
    $\frac{1}{\tlambda_i} \in [\frac{1}{\lambda_i} - \delta, \frac{1}{\lambda_i} + \delta]$, i.e., 
    all the eigenvalues of $\tl^{-1}$
    are within $\delta$ of the eigenvalue
    of $L^{-1}$.
    which therefore implies that the difference between $k^{th}$
    eigenvalues is,
    \begin{equation*}
        \frac{1}{\tlambda_{k}} - \frac{1}{\lambda_{k+1}} 
        \geq \frac{1}{\lambda_{k}} - \frac{1}{\lambda_{k+1}} - \delta
    \end{equation*}
    Since the operators $L^{-1}$, $\tilde{L}^{-1}$ are bounded, 
    the Davis-Kahan $\sin \Theta$ theorem~\citep{davis1970rotation} 
    can be used to conclude that:
    \begin{equation}
        \label{eq:davis_kahan_form_1}
        \|\sin \Theta(\Keigenfunctions, \Keigenfunctionsapprox)\| = 
        \|P_k - \tilde{P}_{k}\| \leq  \frac{\|L^{-1} - \tl^{-1}\|}{\gamma - \delta} \leq \frac{\delta}{\gamma - \delta}
    \end{equation}
    where $\| \cdot \|$ is understood to be the operator norm, 
    and
    $\gamma = \gammavalue$.
    %Further, from \cite{davis1970rotation} 
    %we know that $\|\sin \Theta(V, \tilde{V})\|$ can be written as the following
    %$$\|\sin \Theta(V, \tilde{V})\| = \|P_k - \tilde{P}_{k}\|.$$
    Therefore for any function $g \in \h$
    we have
    \begin{align*}
        \|P_k g - \tilde{P}_k g\|_{\l2} 
        &\leq \|P_k - \tilde{P}_k \|\|g\|_{\l2} \\
        &\leq \frac{ \|L^{-1} - \tl^{-1}\|}{\gamma - \delta}\|g\|_{\l2}
    \end{align*}
    By \eqref{eq:davis_kahan_form_1}, we then get
    $\|P_k g - \tilde{P}_k g\|_{\l2} 
        \leq \frac{\delta}{\gamma - \delta} \|g\|_{\l2}$, which finishes the proof. 
    % Using the fact that for a given orthogonal transformation $O$
    % we have $\|V O - \hat{V}\| = \sqrt{2}\|\sin \Theta\|$,
    % \eqref{eq:davis_kahan_form_1} also implies that there exists an 
    % orthogonal transformation $O: \h \to \h$ such that
    % \begin{equation}
    %     % \|O\tilde{V} - V\|_{\l2} 
    %     \sup_{\substack{a \in \R^k \\ \sum_{i=1}^k a^2_i = 1}}
    %     \left\| \sum_{i=1}^k (O\tvarphi_i - \varphi_i)a_i\right\|_{\l2}
    %     % \leq \frac{2^{3/2}\|L^{-1} - \tl^{-1}\|}{\gamma - \delta} \leq \frac{2^{3/2} \delta}{\gamma - \delta}
    %     \leq \frac{2^{3/2}\|L^{-1} - \tl^{-1}\|}{\gamma - \delta} 
    %         \leq \frac{2^{3/2} \delta}{\gamma - \delta}
    % \end{equation}
\end{proof}

% In the next lemma, we use the result in Lemma~\ref{lemma:perturbation_bound}
% to show that the difference between $\fkeigen$ and $\fkeigenapprox$ 
% is small.
% \input{sections/Proofs/existence_of_tilde_f_span}

Finally, we show that repeated applications of $\tl$ to $\fnn-f$ have also bounded norms: 
\begin{lemma}
[Bounding norms of applications of $\tl$]
\label{lamma:fnn_feigenapprox_close}
The functions $\fnn$ and $f$
satisfy:
%such that $\|\partial^\alpha f - \partial^\alpha \fnn\|_{\l2} \leq \epsilon_{\nn}$
%then we have the following set of inequalities,
\begin{enumerate}
    \item $\|\tl^{n}(\fnn - \fkeigen)\|_{\l2} \leq (n!)^2 \cdot C^n (\epsilon_{\spn} + \epsilon_{\nn})$
    \item $\|\tl^{n}(\fnn - \fkeigenapprox)\|_{\l2} \leq (n!)^2 \cdot C^n (\epsilon_{\spn} + \epsilon_{\nn}) 
        + \constantnumber\left(1 + \constantdeltagamma \right)\lambda_k^n\|\errorf\|_{\l2}$
\end{enumerate}
where
% $\|\cdot \|$ denotes the operator norm of an operator
% and
% $C = \Cvalue$,  
$\delta = \deltavalue$.
% and 
% $\gamma= \gammavalue$.
\end{lemma}
\begin{proof}
    For Part 1, by  Lemma~\ref{lemma:upper_bound_order_n} we have that 
    \begin{equation}
        \label{eq:lemma_15_answer_fnn_fkeigen}
        \|\tl^{n}(\fnn - \fkeigen)\|_{\l2} \leq (n!)^2 \cdot C^n \max_{\alpha:|\alpha|\leq n+2} \|\partial^\alpha (\fnn - \fkeigen)\|_{\l2}
    \end{equation}
    From Assumptions \numassumptions, for any multi-index $\alpha$ we have:
    \begin{align*}
        \|\partial^\alpha \fnn - \partial^\alpha\fkeigen\|_{\l2} 
        &\leq \|\partial^\alpha\fnn - \partial^\alpha f\|_{\l2} 
        + \|\partial^\alpha f - \partial^\alpha \fkeigen\|_{\l2} \\
        &\leq \epsilon_{\nn} + \epsilon_{\spn} \numberthis \label{eq:partial_alpha_fnn_fkeigen}
    \end{align*}
    Combining \eqref{eq:lemma_15_answer_fnn_fkeigen} and \eqref{eq:partial_alpha_fnn_fkeigen} we get
    the result for Part 1.
    
    For Part 2 we have,
    \begin{align}
        \|\tl^{n}(\fkeigenapprox - \fnn)\|_{\l2} 
            &= \|\tl^{n}(\fkeigenapprox - \fkeigen + \fkeigen - \fnn)\|_{\l2}\\
            &\leq  \|\tl^{n}(\fkeigenapprox - \fkeigen )\|_{\l2} 
                + \|\tl^{n}\left(\fkeigen - \fnn\right)\|_{\l2}  \label{eq:twoterp2}
    \end{align}
    Note that from Lemma~\ref{lemma:L_tl_inner_product_close} part (2.) we have
    that $\|L^{-1} \tl - I\|\leq \delta$ (where $\|\cdot\|$ denotes the operator norm). 
    This implies that there exists an operator $\Sigma$, 
    such that $\|\Sigma\| \leq \delta$ and we can express $\tl$ as:
    \begin{align*}
        %& L^{-1} \tl  - I = \Sigma \\
        %\implies& 
        \tl = L (I + \Sigma)
    \end{align*}
    %We now show that there exists a $\tilde{\Sigma}$ such that $\tl^{(n)} = L^{(n)} ( I + \tilde{\Sigma})$ and  $\|\tilde{\Sigma}\| \leq n\errorLntildeLn$. 
    
    We will show that there exists a $\tilde{\Sigma}$, s.t. 
    $\|\tilde{\Sigma}\| \leq n\errorLntildeLn$
    and $\tl^{n} = (I + \tilde{\Sigma})L^{n}$.
    Towards that, we will denote
    $L^{-n} := \underbrace{L^{-1} \circ L^{-1} \circ \cdots L^{-1}}_{\text{n times}}$ 
    and show that 
    \begin{equation}
        \left\|L^{-n}\tl^{n}\right\| \leq  1 + n \errorLntildeLn
        \label{eq:nthpower}
    \end{equation}    
    We have:
    
    \begin{align*}
        \left\|L^{-n}\tl^{n}\right\|
        &= \left\| L^{-n} \left(L(I + \Sigma)\right)^{n}\right\|\\
        &= \left\|L^{-n}\left(L^{n} 
        + \sum_{j=1}^n L^{j-1} \circ (L \circ \Sigma) \circ
            L^{n-j} + \cdots + (L\circ\Sigma)^{n}
            \right) \right\| \\
        &= \left\|I + \sum_{j=1}^n L^{-n} \circ L^{j-1} \circ \Sigma \circ
            L^{n-j} + \cdots + 
            L^{-n} \circ (L\circ\Sigma)^{(n)}
             \right\| \\
        &\stackrel{(1)}{\leq}
        1 + \left\|\sum_{j=1}^n L^{-n} \circ L^{j-1} \circ \Sigma \circ
            L^{n-j}\right\| + \cdots  + \|L^{-n} \circ (L\circ\Sigma)^{n}\|\\
        &\stackrel{(2)}{\leq}
        1 + \sum_{i=1}^n \binom{n}{i} \delta^i \\
        &= (1 + \errorLntildeLnold)^n \\
        &\stackrel{(3)}{\leq}
        e^{n \errorLntildeLnold} \\
        &\leq 1 + 2 n \errorLntildeLnold
    \end{align*}
    where (1) follows from triangle inequality,
    (2) follows from Lemma \ref{lemma:peeling_lemma},
    (3) follows from $1+x \leq e^x$,
    and the last part follows from $n\delta \leq 1/10$ and Taylor expanding $e^x$.
    Next, since $L$ and $\tl$ are elliptic operators, we have
    $\|L^{-n}\tl^{n}\| = \|\tl^{n}L^{-n}\|$. From this, it immediately follows that there exists a $\tilde{\Sigma}$, 
    s.t. $\tl^{n} = (I + \tilde{\Sigma})L^{n}$ with  $\|\tilde{\Sigma}\| \leq n\errorLntildeLn$. 
    
    Plugging this into the first term of \eqref{eq:twoterp2}, we have 
    \begin{align*}
        \|\tl^{n}(\fkeigenapprox - \fkeigen)\|_{\l2} 
        &= \|\tl^{n}\fkeigenapprox - \tl^{n}\fkeigen\|_{\l2} \\
        &= \|\tl^{n}\fkeigenapprox - (I + \tilde{\Sigma})L^{n}\fkeigen\|_{\l2} \\
        &\leq \|\tl^{n}\fkeigenapprox - L^{n}\fkeigen\|_{\l2} 
         + \|\tilde{\Sigma} L^{n} \fkeigen\|_{\l2}\\
          &\leq \|\tl^{n}\fkeigenapprox - L^{n}\fkeigen\|_{\l2} 
         + \|\tilde{\Sigma}\|\|L^{n} \fkeigen\|_{\l2}\\ 
        &\leq \|\tl^{n}\fkeigenapprox - L^{n}\fkeigen\|_{\l2}
          + n\errorLntildeLn\lambda_k^{n}\|\fkeigen\|_{\l2} 
        \numberthis \label{eq:step_1_properties_lemma}
    \end{align*}
    
    The first term in first term in \eqref{eq:step_1_properties_lemma} can be expanded as follows:
    \begin{align*}
    \|\tl^{n}\fkeigenapprox - L^{n}\fkeigen\|_{\l2} &=  \|\tl^{n}\fkeigenapprox - L^n \fkeigenapprox +  L^n \fkeigenapprox+ L^{n}\fkeigen\|_{\l2} \\
    &\leq \|\tl^{n}\fkeigenapprox - L^n \fkeigenapprox\| + \|L^n \fkeigenapprox- L^{n}\fkeigen\|_{\l2} 
    \numberthis \label{eq:step_2_properties_lemma:triangle_ineq}
    \end{align*}
    We'll consider the two terms in turn. 
    
    For the first term, the same proof as that of \eqref{eq:nthpower} shows that there exists an operator $\hat{\Sigma}$, s.t.  $\|\hat{\Sigma}\|\leq 2n\delta$ and 
    $L^n=(I + \hat{\Sigma})\tl^n$. Hence, we have:
    \begin{align*}
    \|\tl^{n}\fkeigenapprox - L^n \fkeigenapprox\| &= \|\tl^{n}\fkeigenapprox - (I+\hat{\Sigma}) \tl^n \fkeigenapprox\|  \\
    &= \|\hat{\Sigma} \tl^n \fkeigenapprox\| \\
    &\leq  \tilde{\lambda}_{k}^n \|\hat{\Sigma}\| \|\fkeigenapprox\|_{\l2} \\
    &\leq  2n\delta \tilde{\lambda}_{k}^n \|f\|_{\l2}  \qquad (\because \|\fkeigenapprox\|_{\l2} \leq \|f\|_{\l2})
    \numberthis \label{eq:step_2_1}
    \end{align*}
    % where we use the result from Lemma~\ref{lemma:power_n_eigenvalue_relation} to get the inequality in $(i)$ along 
    % with the fact that $\|\fkeigenapprox\|_{\l2}\leq f$.

    For the second term in \eqref{eq:step_1_properties_lemma} we have: 
    \begin{align}
        \label{eq:Ln_fkeigenapprox_fkeigen_1}
        \|L^n (\fkeigenapprox-\fkeigen)\|_{\l2} 
        &\leq \sup_{v: v = v_1 - v_2, v_1 \in \Phi_k, v_2 \in \tilde{\Phi}_k} \frac{\|L^n v \|_{\l2}}{\|v\|_{\l2}}   \| \fkeigenapprox-\fkeigen\|_{\l2}
    \end{align}
    To bound the first factor we have:
    \begin{align*}
        \|L^n v \|_{\l2} &= \|L^n (v_1-v_2) \|_{\l2}\\
        &\leq  \|L^n v_1\|_{\l2} + \|L^n v_2\|_{\l2} \\ &= \|L^n v_1\|_{\l2} + \|(I+\hat{\Sigma}) \tl^n v_2\|_{\l2} \\ 
        &\leq \lambda_k^n\|v_1\|_{\l2} +  \tlambda^n_{k} \|I+\hat{\Sigma}\|_2  \|v_2\|_{\l2} \\
        &\leq (\lambda_k^n + \tlambda^n_{k} (1+2n\delta))\|v\|_{\l2}
    \end{align*}
    where we use the fact that $\|v_1\|_{\l2}, \|v_2\|_{\l2} \leq \|v\|_{\l2}$ and $\|\hat{\Sigma}\|\leq 2n\delta$.
    Hence, we can bound 
    \begin{equation}
        \label{eq:Ln_fkeigenapprox_fkeigen_2}
        \sup_{v: v = v_1 - v_2, v_1 \in \Phi_k, v_2 \in \tilde{\Phi}_k} \frac{\|L^n v \|_{\l2}}{\|v\|_{\l2}}
        \leq (\lambda^n_{k} + \tlambda^n_{k} (1+2n\delta) )
    \end{equation}
    From \eqref{eq:Ln_fkeigenapprox_fkeigen_2} and 
    Lemma~\ref{lemma:fkeigen_fkeigenapprox_davis_kahan} we have:
    \begin{align*}
        \|L^n (\fkeigenapprox-\fkeigen)\|_{\l2}  
        &\leq \sup_{v: v = v_1 - v_2, v_1 \in \Phi_k, v_2 \in \tilde{\Phi}_k} \frac{\|L^n v \|_{\l2}}{\|v\|_{\l2}}   \| \fkeigenapprox-\fkeigen\|_{\l2} \\
        &\leq (\lambda^n_{k} + \tlambda^n_{k}(1+2n\delta) )\constantdeltagamma\|f\|_{\l2}
        \numberthis \label{eq:step_2_2}
    \end{align*}
    Therefore from \eqref{eq:step_2_1} and \eqref{eq:step_2_2}, 
    we can upper bound $\|\tl^n(\fkeigenapprox - \fkeigen)\|_{\l2}$ using 
    \eqref{eq:step_1_properties_lemma} as follows:
    \begin{align*}
        \|\tl^n(\fkeigenapprox - \fkeigen)\|_{\l2} 
        &\leq
        \|\tl^n\fkeigenapprox - L^n\fkeigen\|_{\l2}  + 2n\delta \lambda_k^n\|f\|_{\l2}\\
            &\leq 2n\delta\tlambda_k^n  \|f\|_{\l2} 
                + (\lambda_k^n + \tilde{\lambda}^n_k (1+2n\delta) )\constantdeltagamma \|f\|_{\l2} 
                + 2n\delta \lambda_k^n\|f\|_{\l2}\\
            &\stackrel{(i)}{\leq} (1+2n\delta\lambda_k)\left(1
                + (1+2n\delta)\right)\constantdeltagamma \lambda_k^n\|f\|_{\l2}
                + 2n\delta \lambda_k^n\|f\|_{\l2}\\
            %&\stackrel{(ii)}{\leq} \left(2n\delta + (1 + 2n\delta)(1 + 2n\delta \lambda_k)\constantdeltagamma\right) \lambda_k^n\|f\|_{\l2} 
            %    + 2n\delta \lambda_k^n\|f\|_{\l2}\\
            &\stackrel{(ii)}{\leq} 4\left(1 + \constantdeltagamma\right) \lambda_k^n\|f\|_{\l2}
    \end{align*}
    Here in $(i)$ we use the result from Lemma~\ref{lemma:power_n_eigenvalue_relation} and write 
    $\frac{\tlambda^n_k}{\lambda_k^n} \leq 1 + 2n\delta\lambda_k$. % and the fact that $\|\tilde{\Sigma}\| \leq 2n\delta$
    %in $(ii)$. 
    In $(iii)$, we use $n \leq T$ and the fact that $2T\min(1,\lambda_k)\delta \leq 1/10 \leq 1$

    Therefore, finally we have:
    $$\|\tl^{n}\fkeigenapprox - L^{n}\fkeigen\|_{\l2}
    \leq 
    4\left(\constantdeltagamma + 1\right)\lambda_k^n\|\errorf\|_{\l2} $$
    Combining with the result for Part 1,
    Therefore we have the following:
    \begin{align*}
        \|\tl^{n}(\fkeigenapprox - \fnn)\|_{\l2} 
            \leq (n!)^2 \cdot C^n(\epsilon_{\spn} + \epsilon_{\nn})
            + 4\left(1 + \constantdeltagamma \right)\lambda_k^n\|\errorf\|_{\l2}
    \end{align*}
\end{proof}

\section{Technical Lemmas: Manipulating Operators}
Before we state the lemmas we introduce some common notation used throughout this section.
We denote $L^{n} = \underbrace{L \circ L \circ \cdots \circ L}_{\text{n times}}$.
Further we use $L_k$ to denote the operator with $\partial_k a_{ij}$ for all $i, j \in [d]$
and $\partial_k c$ as coefficients, that is:
$$L_k u = \sum_{i,j=1}^d - \left(\partial_k a_{ij}\right) \partial_{ij} u 
        -\sum_{i,j=1}^d \partial_k\left(\partial_i a_i\right)\partial_j u + (\partial_k c)u$$
Similarly the operator $L_{kl}$ is defined as:
$$L_{kl} u = \sum_{i,j=1}^d - \left(\partial_{kl} a_{ij}\right) \partial_{ij} u 
        -\sum_{i,j=1}^d \partial_{kl}\left(\partial_i a_i\right)\partial_j u + (\partial_{kl} c)u$$

\begin{lemma}
    \label{lemma:power_n_eigenvalue_relation}
    Given $\varphi_i$ and $\tvarphi_i$ for all $i \in [k]$ are top $k$ 
    eigenvalues of operators $L$ and $\tl$ respectively, 
    such that $\|L^{-1} - \tl^{-1}\|$ is bounded. 
    Then for all $n \in \N$ we have that 
    $$\tlambda^n_i \leq (1 + \hat{e})\lambda^n_i$$
    where $ i\in [k]$ and $|\hat{e}| \leq 2n\delta \lambda_k$ and 
    $\delta = \deltavalue$.
\end{lemma}
\begin{proof}
    From \eqref{eq:operator_norm_L_tilde_L_inverse} and Weyl's inequality 
    we have for all $i \in \N$
    \begin{align*}
        \label{eq:weyls_inequality}
        \sup_{i} \left|\frac{1}{\lambda_i} - \frac{1}{\tlambda_i}\right| \leq \|L^{-1} - \tl^{-1}\| \leq \delta
    \end{align*}
    %therefore the first $k$ eigenvalues of the operators $L$ and $\tl$
    %are close as well, that is, for all $i \in [k]$ we have the following:
    From this, we can conclude that:
    \begin{align*} &\left|\tlambda_i - \lambda_i \right| \leq \delta \lambda_i \tlambda_i\\
        \implies& \tlambda_i(1 - \delta \lambda_i) \leq \lambda_i \\
        \implies& \tlambda_i \leq \frac{\lambda_i}{(1 - \delta\lambda_i)}\\
        \implies& \tlambda_i \leq (1 + \delta \lambda_i)\lambda_i
        %\implies &\left|\tlambda_i - \lambda_i \right| 
        %    \leq e\lambda_i 
        %\numberthis %\label{eq:eigenvalue_relation}
    \end{align*}
    Writing $\tlambda_i = (1 + \tilde{e}_i)\lambda_i$ (where $\tilde{e}_i = \delta\lambda_i$), we have 
    \begin{align*}
        \left|\tlambda^n_i - \lambda_i^n\right| &= \left|((1 + \tilde{e}_i)\lambda_i)^n - \lambda_i^n\right| \\
        &= \left|\lambda_i^n((1+\tilde{e}_i)^n - 1)\right| \\ 
        &\stackrel{(1)}{\leq} \lambda_i^n |\tilde{e}|_i \left|\sum_{j=1}^n (1+\tilde{e}_i)^j\right|\\
        &\stackrel{(2)}{\leq} \lambda_i^n n |\tilde{e}_i| e^{n |\tilde{e}_i|}\\
        &\stackrel{(3)}{\leq}\lambda_i^n n |\tilde{e}_i| (1+ |2n\tilde{e}_i|) \\
        &\leq 2 \lambda_i^n n |\tilde{e}_i| 
    \end{align*}
    where (1) follows from the factorization $a^n - b^n = (a-b)(\sum_{i=0}^{n-1} a^i b^{n-i-i})$,
    (2) follows from $1+x \leq e^x$,
    and (3) follows from $n |\tilde{e}_i| \leq 1/20$ and Taylor expanding $e^x$. 
    Hence, there exists a $\hat{e}_i$, s.t. $\tlambda^n_i = (1 + \hat{e}_i)\lambda^n_i$
    and $|\hat{e}_i| \leq 2n |\tilde{e}_i|$ (i.e., $|\hat{e}_i| \leq 2n \delta \lambda_i$).
    Using the fact that 
    $\lambda_i \leq \lambda_k$ for all $i\in [k]$
    completes the proof.
\end{proof}

\begin{lemma}[Operator Chain Rule]
    \label{lemma:operator_chain_rule}
    Given an elliptic operator $L$,
    for all $v \in C^{\infty}(\Omega)$ we have the following
    \begin{equation}
        \label{eq:operator_chain_rule_nabla_relation}
        \nabla_k L^{n}u
            = \sum_{i=1}^n \left(L^{n-i}\circ L_k \circ L^{i-1}\right)(u)
            + L^{n}(\nabla_{k} u)
    \end{equation}
    \begin{align}
        \label{eq:operator_chain_rule_nabla_squred_relation}
        \begin{split}
            \nabla_{kl}(L^{n}u)
            &= \sum_{\substack{i, j \\ i < j}}\left(L^{n-i} \circ L_k \circ L^{j-i-1} \circ L_l \circ L^{j-1}\right)u\\
            &+\sum_{\substack{i, j \\ i > j}}\left(L^{n-j} \circ L_k \circ L^{i-j-1} \circ L_l \circ L^{i-1} \right)u \\
            &+ \sum_{i}\left(L^{n-i} \circ L_{kl} \circ L^{i-1}\right)u
            + L^{n}(\nabla_{kl}u)
        \end{split}
    \end{align}
    where we assume that $L^{(0)} = I$.
\end{lemma}
\begin{proof}
   We show the proof using induction on $n$. To handle the base case, for $n = 1$, we have
   \begin{align*}
       \nabla_{k} (Lu) &= \nabla_{k}\left( -\divergence(A\nabla u) + cu\right)\\
       &= \nabla_k \left(-\sum_{ij}a_{ij} \partial_{ij} u - \sum_{ij}\partial_i a_{ij}\partial_j u + cu\right)\\
       &= \left(-\sum_{ij}a_{ij}\partial_{ij}(\partial_k u) - \sum_{ij}\partial_i a_{ij} \partial_j \partial_k u + c \partial_k u \right) \\
        &+ \left(-\sum_{ij}\partial_k a_{ij} \partial_{ij}u - \sum_{ij} \partial_i \partial_k a_{ij} \partial_j u + \partial_k c u\right)\\
       &= L(\nabla_k u) + L_k u \numberthis \label{eq:operator_inequality_beta_1}
   \end{align*}
    Similarly $n=1$ and $k,l \in [d]$, 
   \begin{align*}
       \nabla_{kl} (Lu) &= \nabla_{kl} \left(-\divergence(A\nabla u) + cu\right)\\
       &= \nabla_{kl} \left(-\sum_{ij}a_{ij} \partial_{ij} u - \sum_{ij}\partial_i a_{ij}\partial_j u + cu\right)\\
       &= \left(-\sum_{ij}a_{ij}\partial_{ij}(\partial_{kl} u) - \sum_{ij}\partial_i a_{ij} \partial_j \partial_{kl} u + c \partial_{kl}u \right)\\
        &+\left(-\sum_{ij}\partial_{k} a_{ij} \partial_{ij}\partial_l u 
            - \sum_{ij} \partial_i \partial_{k} a_{ij} \partial_j \partial_l u + \partial_{k} c \partial_l u\right)\\
        &+\left(-\sum_{ij}\partial_{l} a_{ij} \partial_{ij}\partial_k u 
            - \sum_{ij} \partial_i \partial_{l} a_{ij} \partial_j \partial_k u + \partial_{l} c \partial_k u\right)\\
        &+\left(-\sum_{ij}\partial_{kl} a_{ij} \partial_{ij}u - \sum_{ij} \partial_i \partial_{kl} a_{ij} \partial_j u + \partial_{kl} c u\right)\\
       &= L(\nabla_{kl} u) + L_k (\nabla_l u) + L_l (\nabla_k u) + L_{kl} u \numberthis \label{eq:operator_inequality_beta_2}
   \end{align*}
   
For the inductive case, assume that for all $m < n$, \eqref{eq:operator_chain_rule_nabla_relation} and 
   \eqref{eq:operator_chain_rule_nabla_squred_relation} hold.
   Then, for any $k \in [d]$ we have:
   \begin{align*}
       \nabla_k (L^{n}u) &= \nabla_k \left(L \circ L^{n-1}(u)\right)\\
       &= L\left(\nabla_k (L^{n-1} u)\right) + L_k\left(L^{n-1}u\right)\\
       &= L\left(\sum_{i=1}^{n-1} \left(L^{n-1-i}\circ L_{k} \circ L^{i-1}\right)u + L^{n-1}(\nabla_k u)\right) + L_k\left(L^{n-1}\right)u\\
       &= \sum_{i=1}^n \left(L^{n-i}\circ L_k \circ L^{i-1}\right)(u)
        + L^{n}(\nabla_k u) \numberthis \label{eq:operator_inequality_final_beta_1}
   \end{align*}
   Similarly,  for all $k,l \in [d]$ we have:
   \begin{align*}
       \nabla_{kl} (L^{n}u) &= \nabla_{kl} \left(L \circ L^{n-1}(u)\right)\\
       &= L\left(\nabla_{kl} (L^{n-1} u)\right) 
       + L_k\left(\nabla_l\left(L^{n-1}u\right) \right) 
       + L_l\left(\nabla_k\left(L^{n-1}u\right) \right) 
       +  L_{kl}\left(L^{n-1}u\right)  \\
        &= L\Bigg(\sum_{\substack{i, j \\ i < j}}^{n-1}\left(L^{n-1-i} \circ L_k \circ L^{j-i-1} \circ L_l \circ L^{j-1}\right)u\\
        &\qquad\; +\sum_{\substack{i, j \\ i > j}}^{n-1}\left(L^{n-1-j} \circ L_k \circ L^{i-j-1} \circ L_l \circ L^{i-1} \right)u \\
        &\qquad\; + \sum_{i=1}^{n-1}\left(L^{n-1-i} \circ L_{kl} \circ L^{i-1}\right)u
        + L^{n-1}(\nabla_{kl}u) \Bigg) \\
        & + L_k \Bigg( \sum_{i=1}^{n-1} \left(L^{n-1-i}\circ L_l \circ L^{i-1}\right)(u) 
        + L^{n-1}(\nabla_l u) \Bigg)
            \qquad \text{(from \eqref{eq:operator_inequality_final_beta_1})}
        \\
        & + L_l \Bigg( \sum_{i=1}^{n-1} \left(L^{n-1-i}\circ L_k \circ L^{i-1}\right)(u)
        + L^{n-1}(\nabla_k u) \Bigg)
            \qquad \text{(from \eqref{eq:operator_inequality_final_beta_1})}
        \\
       &+  L_{kl}\left(L^{n-1}u\right)  \\
        &= \sum_{\substack{i, j \\ i < j}}^n\left(L^{n-i} \circ L_k \circ L^{j-i-1} \circ L_l \circ L^{j-1}\right)u\\
        &\qquad +\sum_{\substack{i, j \\ i > j}}^n\left(L^{n-j} \circ L_k \circ L^{i-j-1} \circ L_l \circ L^{i-1} \right)u \\
        &\qquad +\sum_{i}^n\left(L^{n-i} \circ L_{kl} \circ L^{i-1}\right)u
        + L^{n}(\nabla_{kl}u)
        \numberthis \label{eq:operator_inequality_final_beta_2}
   \end{align*}
    By induction, the claim follows.
\end{proof}

\begin{lemma}
    \label{lemma:upper_bound_order_1}
    For all $u \in C^{\infty}(\Omega)$
    then for all $k, l \in [d]$ the following upper bounds hold,
    \begin{equation}
        \label{eq:Lu_C_upper_bound}
        \|Lu\|_{\l2} \leq C \max_{\alpha:|\alpha| \leq 2} \|\partial^\alpha u\|_{\l2}
    \end{equation}
    \begin{equation}
        \label{eq:nabla_k_Lu_C_upper_bound}
        \|\nabla_k (Lu)\|_{\l2} \leq 2 \cdot C \max_{\alpha:|\alpha| \leq 3} \|\partial^\alpha u\|_{\l2}
    \end{equation}
    and
    \begin{equation}
        \label{eq:nabla_square_k_Lu_C_upper_bound}
        \|\nabla_{kl} (Lu)\|_{\l2} \leq 4 \cdot C \max_{\alpha:|\alpha| \leq 4} \|\partial^\alpha u\|_{\l2}
    \end{equation}
    where 
    $$C := \Cvalue.$$
\end{lemma}
\begin{proof}
    We first show the upper bound on $\|Lu\|_{\l2}$:
    \begin{align*}
        \|Lu\|_{\l2}
        &\leq \left\|-\sum_{i,j=1}^d a_{ij}\partial_{ij}u - \sum_{i,j=1}^d \partial_{i}a_{ij}\partial_j u + cu\right\|_{\l2}\\
        &\leq^{(1)} \underbrace{(2d^2 + 1)\max\left\{\max_{i,j}\|\partial_{i}a_{ij}\|_{\linf}, \max_{i,j}\|a_{ij}\|_{\linf}, \|c\|_{\linf}\right\}}_{C_1}
            \max_{\alpha:|\alpha|\leq 2}\|\partial^\alpha u\|_{\l2}\\
        &\leq C_1
            \max_{\alpha:|\alpha|\leq 2}\|\partial^\alpha u\|_{\l2}
        \numberthis \label{eq:Lu_C_1_upper_bound}
    \end{align*}
    where (1) follows by H\"older.
    
    Proceeding to $\|\nabla_k (Lu)\|_{\l2}$, from Lemma~\ref{lemma:upper_bound_order_n} we have
    %we use the result from \eqref{eq:operator_inequality_beta_1},
    {\small
    \begin{align*}
        \|\nabla_{k}(Lu)\|_{\l2} &\leq \|L_{k}u\|_{\l2} + \|L(\nabla_{k}u)\|_{\l2}\\
        &\leq \left\|-\sum_{i,j=1}^d\partial_{k}a_{ij}\partial_{ij}u - \sum_{i,j=1}^d\partial_{ik}a_{ij}\partial_{j}u + \partial_{k}cu\right\|_{\l2}\\
            &+\left\| -\sum_{i,j=1}^da_{ij}\partial_{ijk}u - \sum_{i,j=1}^d\partial_ia_{ij} \partial_{jk}u + c\partial_{k}u\right\|_{\l2}\\
        &\leq (2d^2 + 1)\max\left\{\max_{\alpha:|\alpha|\leq 2}\max_{i,j}\|\partial^{\alpha}a_{ij}\|_{\linf}, \|\partial_{k} c\|_{\linf}\right\}
                \max_{\alpha:|\alpha|\leq 2}\|\partial^\alpha u\|_{\l2}\\
        &+ (2d^2 + 1)\max\left\{\max_{\alpha:|\alpha|\leq 1}\max_{i,j}\|\partial^{\alpha}a_{ij}\|_{\linf}, \|c\|_{\linf}\right\}
                \max_{\alpha:|\alpha|\leq 3}\|\partial^\alpha u\|_{\l2}\\
        \implies  \|\nabla_{k}(Lu)\|_{\l2}
        &\leq 2 \cdot \underbrace{(2d^2 + 1)\max\left\{\max_{\alpha:|\alpha|\leq 2}\max_{i,j}\|\partial^{\alpha}a_{ij}\|_{\linf},
                \max_{\alpha:|\alpha|\leq 1}\|\partial^\alpha c\|_{\linf}\right\}}_{C_2}
                \max_{\alpha:|\alpha|\leq 3}\|\partial^\alpha u\|_{\l2}\\
        &\leq 2 \cdot C_2 
                \max_{\alpha:|\alpha|\leq 3}\|\partial^\alpha u\|_{\l2}
        \numberthis \label{eq:nabla_k_Lu_C_2_upper_bound}
    \end{align*}
    }
    We use the result from Lemma~\ref{lemma:operator_chain_rule} (equation 
    \eqref{eq:operator_inequality_beta_2}), 
    to upper bound the quantity $\|\nabla_{kl}(Lu)\|_{\l2}$
    {\small
    \begin{align*}
        \|\nabla_{kl}(Lu)\|_{\l2} &\leq \|L_{kl}u\|_{\l2} + \|L_k(\nabla_lu)\|_{\l2} + \|L_l(\nabla_k u)\|_{\l2} + \|L(\nabla_{kl}u)\|_{\l2}\\
        &\leq \left\|-\sum_{i,j=1}^d\partial_{kl}a_{ij}\partial_{ij}u - \sum_{i,j=1}^d\partial_{ikl}a_{ij}\partial_{j}u + \partial_{kl}cu\right\|_{\l2}\\
            &+ \left\| -\sum_{i,j=1}^d\partial_k a_{ij}\partial_{ij}\partial_lu - \sum_{i,j=1}^d\partial_i\partial_k a_{ij} \partial_{j}\partial_lu + \partial_k c\partial_{l}u\right\|_{\l2}\\
            &+ \left\| -\sum_{i,j=1}^d\partial_l a_{ij}\partial_{ij}\partial_ku - \sum_{i,j=1}^d\partial_i\partial_l a_{ij} \partial_{j}\partial_ku + \partial_l c\partial_{k}u\right\|_{\l2}\\
            &+ \left\| -\sum_{i,j=1}^da_{ij}\partial_{ijkl}u - \sum_{i,j=1}^d\partial_ia_{ij} \partial_{jkl}u + c\partial_{kl}u\right\|_{\l2}\\
        &\leq (2d^2 + 1)\max\left\{\max_{\alpha:|\alpha|\leq 3}\max_{i,j}\|\partial^{\alpha}a_{ij}\|_{\linf}, \|\partial_{kl} c\|_{\linf}\right\}
                \max_{\alpha:|\alpha|\leq 2}\|\partial^\alpha u\|_{\l2}\\
        &+ 2(2d^2 + 1)\max\left\{\max_{\alpha:|\alpha|\leq 2}\max_{i,j}\|\partial^{\alpha}a_{ij}\|_{\linf}, \|c\|_{\linf}\right\}
                \max_{\alpha:|\alpha|\leq 3}\|\partial^\alpha u\|_{\l2}\\
        &+ (2d^2 + 1)\max\left\{\max_{\alpha:|\alpha|\leq 2}\max_{i,j}\|\partial^{\alpha}a_{ij}\|_{\linf}, \|c\|_{\linf}\right\}
                \max_{\alpha:|\alpha|\leq 4}\|\partial^\alpha u\|_{\l2}\\
        \implies  \|\nabla_{kl}(Lu)\|_{\l2}
        &\leq 4 \cdot \underbrace{(2d^2 + 1)\max\left\{\max_{\alpha:|\alpha|\leq 3}\max_{i,j}\|\partial^{\alpha}a_{ij}\|_{\linf},
                \max_{\alpha:|\alpha|\leq 2}\|\partial^\alpha c\|_{\linf}\right\}}_{C_3}
                \max_{\alpha:|\alpha|\leq 4}\|\partial^\alpha u\|_{\l2}\\
        &\leq 4 \cdot C_3 
                \max_{\alpha:|\alpha|\leq 4}\|\partial^\alpha u\|_{\l2}
        \numberthis \label{eq:nabla_kl_Lu_C_3_upper_bound}
    \end{align*}
    }
    Since $C_1 \leq C_2 \leq C_3$, we define $C:=C_3$ and therefore
    from equations \eqref{eq:Lu_C_1_upper_bound},
         \eqref{eq:nabla_k_Lu_C_2_upper_bound} and \eqref{eq:nabla_kl_Lu_C_3_upper_bound}
    the claim follows.
    
    Further, we note that from \eqref{eq:nabla_k_Lu_C_2_upper_bound}, we also have that
    \begin{equation}
        \label{eq:upper_bound_first_order_L_terms}
        \|L_k(u)\|_{\l2}, \|L(\nabla_k u)\|_{\l2} \leq C \max_{\alpha: |\alpha|\leq 3}\|\partial^\alpha u\|_{\l2}
    \end{equation}
    and similarly from \eqref{eq:nabla_kl_Lu_C_3_upper_bound} we have that,
    \begin{equation}
        \label{eq:upper_bound_second_order_L_terms}
        \|L_{kl}(u)\|_{\l2}, \|L_{k}(\nabla_l u)\|_{\l2},
        \|L_{l}(\nabla_k u)\|_{\l2}, \|L(\nabla_{kl}u)\|_{\l2}
        \leq C \max_{\alpha: |\alpha|\leq 4}\|\partial^\alpha u\|_{\l2}
    \end{equation}
\end{proof}

\begin{lemma}
    \label{lemma:upper_bound_order_n}
    For all $u \in C^{\infty}(\Omega)$
    and $k, l \in [d]$
    then for all $n \in \N$ we have the following upper bounds,
    \begin{equation}
        \|L^{n}u\|_{\l2} \leq (n!)^2 \cdot C^n \max_{\alpha: |\alpha| \leq n + 2} \|\partial^\alpha u\|_{\l2}
    \end{equation}
    \begin{equation}
        \|\nabla_k (L^{n}u)\|_{\l2} \leq (n+1) \cdot (n!)^2 \cdot C^n \max_{\alpha: |\alpha| \leq n + 2} \|\partial^\alpha u\|_{\l2}
    \end{equation}
    \begin{equation}
        \|\nabla_{kl} (L^{n}u)\|_{\l2} \leq ((n+1)!)^2 \cdot C^n \max_{\alpha: |\alpha| \leq n + 3} \|\partial^\alpha u\|_{\l2}
    \end{equation}
    where $C = \Cvalue$.
\end{lemma}
\begin{proof}
    We prove the Lemma by induction on $n$. The base case $n=1$ follows from Lemma~\ref{lemma:upper_bound_order_1}, along with the fact that $\max_{\alpha:|\alpha|\leq 2}\|\partial^\alpha u\|_{\l2} \leq \max_{\alpha:|\alpha|\leq 3}\|\partial^\alpha u\|_{\l2}$.
    
    To show the inductive case, assume that the claim holds for all $m \leq (n-1)$.
    By Lemma \ref{lemma:upper_bound_order_1}, we have
    \begin{align*}
        \|L^{n}u\|_{\l2} &= \|L(L^{n-1}u)\|_{\l2} \\
        &\leq \left\|-\sum_{i,j=1}^d a_{ij}\partial_{ij}(L^{n-1}u) 
            - \sum_{i,j=1}^d \partial_{i}a_{ij}\partial_j (L^{n-1}u) + c(L^{n-1}u)\right\|_{\l2}\\
        &\leq C \cdot \max\left\{\|L^{n-1}u\|_{\l2}, \max_{i} \|\nabla_{i} (L^{n-1}u)\|_{\l2}, \max_{i,j}\|\nabla_{ij}(L^{n-1}u)\|_{\l2}\right\}\\
        &\leq C\cdot (n!)^2 \cdot C^{n-1} \max_{\alpha:|\alpha|\leq (n-1)+3}\|\partial^\alpha u\|_{\l2}
        \end{align*} 
        Thus, we have
        $$ \|L^{n}u\|_{\l2} 
        \leq (n!)^2 \cdot C^{n} \max_{\alpha:|\alpha|\leq n+2}\|\partial^\alpha u\|_{\l2}
        $$ 
    as we need. 
    
    Similarly, for $k \in [d]$, we have:
    \begin{align*}
        \|\nabla_{k}(L^{n}u)\|_{\l2} 
            &\leq \sum_{i=1}^n \left\|\left(L^{n-i}\circ L_k \circ L^{i-1}\right)(u)\right\|_{\l2} + \|L^{n}(\nabla_k u)\|_{\l2}\\
            &\leq (n) \cdot (n!)^2 \cdot C^n \max_{\alpha: |\alpha|\leq n+2}\|\partial^\alpha u\|_{\l2} 
                + (n!)^2 \cdot C^n \max_{\alpha:|\alpha|\leq n+2} \|\partial^\alpha u\|_{\l2}\\
            &\leq (n+1) \cdot (n!)^2 \cdot C^n \max_{\alpha: |\alpha|\leq n+2}\|\partial^\alpha u\|_{\l2} 
            \numberthis \label{upper_bound_n_nabla_final}
    \end{align*}
    Finally, for $k,l \in [d]$ we have
    \begin{align*}
        \|\nabla_{kl}(L^{n}u)\|_{\l2} 
        &\leq \sum_{\substack{i, j \\ i < j}}\left\|\left(L^{n-i} \circ L_k \circ L^{j-i-1} \circ L_l \circ L^{j-1}\right)u\right\|_{\l2}\\
        &+\sum_{\substack{i, j \\ i > j}}\left\|\left(L^{n-j} \circ L_k \circ L^{i-j-1} \circ L_l \circ L^{i-1} \right)u\right\|_{\l2}\\
        &+ \sum_{i}\left\|\left(L^{n-i} \circ L_{kl} \circ L^{i-1}\right)u\right\|_{\l2}
        + \|L^{n}(\nabla_{kl}u)\|_{\l2}\\
        &\leq n(n +1) \cdot (n!)^2 \cdot C^n \max_{\alpha:|\alpha|\leq n+2}\|\partial^\alpha u\|_{\l2} \\
        &+ n \cdot (n!)^2 \cdot C^n \max_{\alpha:|\alpha|\leq n+2}\|\partial^\alpha u\|_{\l2} 
        + C^n \max_{\alpha:|\alpha|\leq n+3}\|\partial^\alpha u\|_{\l2} \\
        \implies 
        \|\nabla_{kl}(L^{n}u)\|_{\l2} 
        &\leq \left((n+1)!\right)^2 \cdot C^n \max_{\alpha:|\alpha|\leq n+3}\|\partial^\alpha u\|_{\l2} \\
        \numberthis \label{upper_bound_n_nabla_square_final}
    \end{align*}
    %Therefore from \eqref{eq:upper_bound_n_Lu_final}, \eqref{upper_bound_n_nabla_final} and \eqref{upper_bound_n_nabla_square_final}
    %the claim follows.
    Thus, the claim follows.
\end{proof}

\begin{lemma}
    \label{lemma:peeling_lemma}
    Let $A_n^{i}$,  $i \in [n]$ be defined as a composition of 
    $(n-i)$ applications of $L$ and $i$ applications of $L \circ \Sigma$ (in any order), s.t. $\|\Sigma\| \leq \delta$. Then, we have: 
    \begin{equation}
        \label{eq:LinvA_lemma_main_inequality}
        \|L^{-n}A_n^{i}\| \leq  \delta^i
    \end{equation}
\end{lemma}
\begin{proof}
    We prove the above claim by induction on $n$.
    
    For $n = 1$ we have two cases. If $A^{(1)} = L \circ \Sigma$, we have:
    $$ \| L^{-1} \circ L \circ \Sigma \| \leq \delta$$
    If $A^{(1)} = L$ we have:
    $$ \| L^{-1} L \| = 1$$ 
    
    Towards the inductive hypothesis, assume that for $m \leq n-1$ and $i \in [n-1]$ it holds that,
    \begin{align*}
        \|L^{n-1}A_{n-1}^{i}\| \leq  \delta^i
    \end{align*}
    
    For $n$, we will have two cases.
    First, if $A_{n}^{i+1} = A_{n-1}^{i} \circ L \circ  \Sigma$, %here the number of $\Sigma$s has increased to $i+1$
    %in the following case 
    by submultiplicativity of the operator norm, as well as the fact that similar operators have identical spectra (hence equal operator norm) we have:
    \begin{align*}
        \|L^{-n} \circ A^{i+1}_n\|  
        &= \|L^{-1} \circ L^{-(n-1)} \circ A_{n-1}^{(i)} \circ L \circ  \Sigma\|\\
        &= \|L^{-(n-1)} \circ A_{n-1}^{i} \circ L \circ  \Sigma \circ L^{-1}\|\\
            &\leq \delta \|L^{-(n-1)}A_{(n-1)}^{i-1}\|\|L\circ \Sigma \circ L^{-1}\| \\
            &\leq \delta^{i} \delta = \delta^{i+1}
    \end{align*}
    so the inductive claim is proved. 
    In the second case, $A_{n}^{i} = A_{n-1}^{i}L$ and we have, by using the fact that the similar operators have identical spectra:
    \begin{align*}
        \|L^{-n} \circ A_{n}^{i} \circ L\| &= \|L^{-(n-1)}\circ A_{n-1}^{i}\circ L  \circ L^{-1}\|\\
        &= \|L^{-(n-1)} \circ A_{n-1}^{i}\| 
        \leq \delta^i
    \end{align*}
    where the last inequality follows by the inductive hypothesis.
\end{proof}

\end{document}